\documentclass{amsart}
\usepackage[margin=1.4in]{geometry}

\PassOptionsToPackage{numbers, compress}{natbib}

\usepackage{amsmath,amsfonts,bm}
\usepackage{amssymb}

\def\eqref#1{equation~\ref{#1}}

\def\1{\bm{1}}

\def\vq{{\bm{q}}}
\def\vr{{\bm{r}}}
\def\vs{{\bm{s}}}

\def\vu{{\bm{u}}}
\def\vv{{\bm{v}}}
\def\vw{{\bm{w}}}
\def\vx{{\bm{x}}}
\def\vy{{\bm{y}}}
\def\vz{{\bm{z}}}

\DeclareMathAlphabet{\mathsfit}{\encodingdefault}{\sfdefault}{m}{sl}
\SetMathAlphabet{\mathsfit}{bold}{\encodingdefault}{\sfdefault}{bx}{n}

\def\sB{{\mathbb{B}}}

\def\sI{{\mathbb{I}}}

\def\sU{{\mathbb{U}}}
\def\sV{{\mathbb{V}}}

\newcommand{\R}{\mathbb{R}}

\def\cA{{\mathcal{A}}}
\def\cB{{\mathcal{B}}}
\def\cC{{\mathcal{C}}}
\def\cD{{\mathcal{D}}}
\def\cE{{\mathcal{E}}}
\def\cF{{\mathcal{F}}}
\def\cG{{\mathcal{G}}}

\def\cJ{{\mathcal{J}}}
\def\cK{{\mathcal{K}}}

\def\cM{{\mathcal{M}}}
\def\cN{{\mathcal{N}}}
\def\cO{{\mathcal{O}}}

\def\cS{{\mathcal{S}}}

\def\cU{{\mathcal{U}}}
\def\cV{{\mathcal{V}}}

\def\cY{{\mathcal{Y}}}

\allowdisplaybreaks

\usepackage{amsthm}

\newtheorem{theorem}{Theorem}[section]
\newtheorem{lemma}[theorem]{Lemma}

\theoremstyle{definition}
\newtheorem{definition}[theorem]{Definition}

\theoremstyle{remark}

\usepackage[utf8]{inputenc} 
\usepackage[T1]{fontenc}    
\usepackage{hyperref}       
\usepackage{url}            
\usepackage{booktabs}       
\usepackage{amsfonts}       
\usepackage{nicefrac}       
\usepackage{microtype}      
\usepackage{xcolor}         

\usepackage{algorithm}
\usepackage{algpseudocode}
\usepackage{amsmath}
\usepackage{graphicx}
\usepackage{listings}
\usepackage{enumitem}
\usepackage{float}
\usepackage{multirow}
\usepackage{amsrefs}

\makeatletter
\renewcommand{\eqref}[1]{\textup{(\ref{#1})}}
\makeatother

\title[Set-Valued Solution Map Learning]{Bifurcation Models: Learning Set-Valued Solution Maps with Weight-Tied Dynamics}

\author{Caleb Jore}
\address{(CJ) School of Data, Mathematical, and Statistical Sciences, University of Central Florida, Orlando, FL 32826.}
\email{calebjore@gmail.com}
\author{Jialin Liu}
\address{(JL) School of Data, Mathematical, and Statistical Sciences, University of Central Florida, Orlando, FL 32826.}
\email{jialin.liu@ucf.edu}
\date{\today}

\thanks{The authors are listed alphabetically. Corresponding author: Jialin Liu, jialin.liu@ucf.edu.}

\begin{document}

\begin{abstract}
Many scientific and combinatorial problems admit multiple correct solutions, not a single label. Standard supervised learning resolves this ambiguity by choosing one solution as the target, but this hidden selector can be arbitrary, discontinuous, and harder to learn than the underlying solution set. We study bifurcation models, a weight-tied dynamical view in which different initializations can converge to different stable equilibria, so the model represents an attractor landscape rather than one chosen branch. We prove that broad set-valued maps with locally Lipschitz branches can be represented by regular equilibrium dynamics and that the induced selectors are almost everywhere regular, while manual selectors can be arbitrarily irregular. Experiments on frustrated Ising models show that such dynamics can discover multiple valid equilibria without branch labels and outperform single-branch supervision. Allen--Cahn experiments further show that diversity is not automatic: it can be encouraged explicitly, but with an accuracy--diversity tradeoff.
\end{abstract}

\maketitle

\section{Introduction}

Most supervised learning formulations begin by assuming that the target is a function: given an input $\vx$, there is a single desired output $\vy_\star=f(\vx)$. This assumption is so standard that it is often invisible. Yet in many scientific and engineering problems, it is not the right abstraction. A frustrated spin system may admit many low-energy configurations \cite{kim2010quantum}; a nonlinear PDE may have multiple steady states \cite{uecker2021continuation}; an inverse-design problem may have several physically valid designs \cite{you2024deep}.
In such settings, the object to be learned is not a single-valued map $f:\cK\to \cY$, but a set-valued map 
\[
\cF:\cK \rightrightarrows \cY, \qquad \vx \mapsto \cF(\vx), 
\]

where $\cF(\vx)$ is the set of admissible solutions for the same input $\vx$.

The usual supervised workaround is to choose one element
$u(x)\in \mathcal F(x)$ and train a model to imitate this branch. This is benign when a regular selector is known, but in many multi-solution problems no such selector is specified: labels may come from solver initialization \cite{farrell2015deflation,zou2025learning}, a symmetry-breaking convention \cite{lawrence2024improving}, or an arbitrary choice among valid solutions. The labeling procedure therefore defines a hidden selector, which may switch between branches as $x$ varies and inject artificial irregularity into the learning problem.
An alternative is to learn the target as set-valued. Neural  architectures for sets can represent and process unordered collections \cite{qi2017pointnet,zaheer2017deep}. In multi-solution scientific problems, however, the target is not merely an unordered list but a solution set whose elements may correspond to distinct physical, geometric, or combinatorial branches. This raises a basic question: \textit{should multi-solution problems be represented by choosing one branch, or by learning the solution set itself?}

We study this question through the lens of weight-tied parameterized dynamics. Starting from an initial state $\vy_0$ that may be fixed, sampled, or chosen by the user, the model iterates
\[
\vy_{t+1}=g_\theta(\vy_t,\vx), \qquad t=0,1,2,\ldots .
\]
Here, $g_\theta$ may itself be a multi-layer neural module, such as an MLP, GNN, or neural operator. The term \emph{\textbf{weight-tied}} means that, when the dynamics are unrolled over time, the same parameters $\theta$ are reused at every iteration $t$, rather than using a different layer at each depth.

For a fixed input $\vx$, different initializations $\vy_0$ may converge to different stable equilibria. Thus the model does not need to output a single label. Instead, it can represent the attractor set 
\begin{equation}
\label{eq:attractor-theta}
    \cA_\theta(\vx):= 
    \left\{ 
    \lim_{t\to\infty} g_\theta^t(\vy_0,\vx) \;:\; \vy_0\in \cY
    \right\},
\end{equation}

where $g^t(\vy_0,\vx)$ denotes applying the update $g(\cdot,\vx)$ for $t$ steps
starting from $\vy_0$. For the same input $\vx$, trajectories from different initial states $\vy_0$ may split over iterations into different basins of attraction and converge to different solution branches. We refer to this viewpoint as a \textbf{\textit{bifurcation model}}. 

Rather than introducing a new architecture, we focus on three fundamental questions:
\begin{itemize}
    \item[(Q1)] \textbf{Representation.} Can weight-tied models represent set-valued mappings in a regular, learnable way? More precisely, for what class of set-valued maps $\cF:\cK\rightrightarrows \cY$ does there exist an operator $g:\cY\times \cK\to \cY$ whose stable fixed points recover $\cF(\vx)$?
    \item[(Q2)] \textbf{Label selection.} If such a representation exists, and suppose all solution branches are available during training. Should we learn the whole set-valued map, or select one element of $\cF(\vx)$ as the label and train a single-valued predictor? What regularity or stability can be lost when a set-valued problem is collapsed into a manually chosen solution selector?
    \item[(Q3)] \textbf{Discovery without complete branch information.} In practice, the full set of branches is often not observed. Given only an unsupervised physical objective and multiple initializations, will a weight-tied dynamic model automatically discover multiple solutions without full branch information during training? 
\end{itemize}
Centered around the three questions, we contribute the following:
\begin{itemize}
    \item \textbf{Representation.} We prove that broad set-valued maps with locally Lipschitz branches can be represented by regular weight-tied equilibrium dynamics: for almost every initialization, trajectories converge to valid branches, and varying the initialization recovers all branches. 
\item \textbf{Selection.} We show that selecting one element of $\cF(\vx)$ as the label can create arbitrarily irregular supervised targets, even when the underlying branches are regular. In contrast, the selectors induced by equilibrium dynamics are locally Lipschitz almost everywhere. 
\item \textbf{Discovery.} We show that in favorable cases, weight-tied dynamics can discover multiple valid equilibria even without complete branch information during training. In addition, we observe that such multi-branch discovery is problem-dependent rather than automatic. It emerges naturally in our Ising experiments, but in Allen--Cahn energy minimization can collapse to a dominant branch despite the model's multi-equilibrium capacity. Explicit diversity promotion recovers multiple branches, revealing an accuracy--diversity tradeoff.
\end{itemize}


Here we discuss four lines of related works. The first line of work involves implicit or equilibrium models (e.g., \cite{bai2019deep,el2021implicit,fung2022jfb,gilton2021deep,geng2021training}). While these models usually focus on designing network layers that converge to a single, unique fixed point for deterministic prediction \cite{el2021implicit,winston2020monotone,jafarpour2021robust,revay2020lipschitz,havens2023exploiting,liu2026expressive}, here we treat the model as a dynamical system explicitly capable of representing an entire set of valid solutions through multiple stable equilibria. 
The second line of work is unrolling networks (e.g., \cite{monga2021algorithm,gregor2010learning,yang2016deep,xin2016maximal,metzler2017learned,zhang2018ista,adler2018learned,liu2019alista}), including recurrent networks \cite{bansal2022end,wang2025hierarchical}, looped transformers \cite{giannou2023looped,geiping2025scaling}, etc. We share their computational realization of unrolling weight-tied iterations during training, but our focus is fundamentally different: rather than optimizing a finite-depth refinement map, we study the limiting attractor structure of the shared operator to capture multi-basin landscapes. 
The third line of related work, \cite{li2023learning}, also uses an iterated map to recover multiple optima, but specifically aims to emulate the proximal-point algorithm for an explicit objective. In contrast, we study weight-tied dynamics more broadly as a fundamental representation of set-valued solution maps, analyzing their representational capacity, label regularity, and unsupervised branch discovery.
Finally, \cite{zou2025learning} uses deep ensembles of PINNs to discover multiple PDE solutions, which requires training an independent ensemble for each individual problem instance. Conversely, we learn an amortized solver over a distribution of instances, where solution multiplicity naturally emerges from different state initializations within a single shared weight-tied model.

The remainder of this paper is organized as follows. In Section \ref{sec:theory}, we analyze weight-tied dynamics within a nonparametric setting so that our theoretical results capture the generic nature of dynamical representations, not restricted to specific neural parameterizations. Sections \ref{sec:ising} and \ref{sec:allen-cahn} instantiate the same underlying principle using different parameterized operators in concrete applications.

\section{Theoretical Foundations}
\label{sec:theory}


In this section, we address (Q1) and (Q2) theoretically. We consider finite-branch set-valued maps $\cF(\vx)=\{f_1(\vx),\cdots,f_n(\vx)\}$, allowing different branches to coincide at some inputs.

\subsection{Can Weight-Tied Dynamics Represent Set-Valued Maps?}
\label{sec:expressivity}

To isolate the fundamental capacity of weight-tied dynamics from specific network structures, we study this in a generic nonparametric setting; thus, we drop the parameter subscript $\theta$ and study the existence/properties of general operators $g$. Accordingly, the attractor set in \eqref{eq:attractor-theta} can be written as 
\begin{equation}
\label{eq:attractor-set-g}
    \mathcal A_g(\vx) := \left\{ \lim_{t\to\infty} g^t(\vy_0,\vx) : \vy_0\in\mathcal Y
    \right\}.
\end{equation}

We now formalize (Q1) in this context: \textit{Let
$\mathcal F$ be a set-valued map. Does there exist $g$ such that its attractor set exactly recover the target, i.e., $\mathcal A_g(\vx)=\mathcal F(\vx)$, for all $\vx\in \cK$? }

In fact, a naive answer satisfies the above condition. Consider the following $g$ with $0<\eta<1$: 
\begin{equation}
\label{eq:naive}
g(\vy,\vx) := (1-\eta) \vy + \eta \Pi_{\cF(\vx)}(\vy)
\end{equation}

where $\Pi_{\cF(\vx)}(\vy)$ denotes a projection of $\vy$ onto $\cF(\vx)$, with an arbitrary fixed tie-breaking rule when the projection is nonunique.
Because this operator explicitly forces $\vy$ toward the closest point in the target set, any trajectory will unconditionally converge to a branch in $\mathcal{F}(\vx)$, and sweeping across all initializations $\vy_0$ will successfully recover every branch.

However, this naive construction is not practically insightful. Let's consider an example: $\cK=\cY=\R$ and $\cF(x) = \{x+1, x-1\}$ with two branches $f_1(x)=x+1$ and $f_2(x)=x-1$. Then \eqref{eq:naive} becomes
\[ g(y,x) = (1-\eta) y + \eta (x+1) \sI(y \geq x) + \eta (x-1) \sI(y<x) \]

Now we fix $y=0$. It can be shown that $g(0,x)$ is not Lipschitz continuous near $x=0$:
\[ \lim_{x \to 0^{+}} g(0,x) = \eta , \quad \lim_{x \to 0^{-}} g(0,x) = -\eta , \quad \lim_{\delta \to 0} \sup_{x_1,x_2 \in \sB(0,\delta)} \frac{|g(0,x_1)-g(0,x_2)|}{|x_1-x_2|} = \infty\]

Therefore, $(0,0)$ is a singular point of $g(0,x)$. In fact, all the points on the switching boundaries $\{(x,y):y=x\}$ are singular. This reveals a key limit of this naive construction: \textit{it's hard to parameterize such an operator $g$ with neural networks.} While neural networks have the ability  to approximate operators with singular points, doing so demands much higher network complexity—i.e., increasing depth/width to capture the growing steepness near the singularity (see, e.g., \cite{telgarsky2017neural}).

Therefore, (Q1) becomes: \textit{whether, given $\cF$, there exists an operator $g$ that satisfies the above conditions while also being regular (specifically, Lipschitz continuous) with respect to the input}. This regularity is essential, as it makes $g$ an easier operator for a neural network to model and parameterize. 

Our findings show an affirmative result, provided we exclude the \textit{unstable collapse set} $\mathcal{U}$: a specific subset of inputs where branch intersections are not locally constant.

\begin{definition}[Stable and Unstable Collapse]
\label{def:main-stable-collapse}
Let $f_1, \dots, f_n : \mathcal{K} \to \mathcal{Y}$ be branch maps. A point $\vx \in \mathcal{K}$ exhibits a \textit{stable collapse pattern} if there exists a neighborhood around $\vx$ wherein the equality relationship between any pair of branches $f_i$ and $f_j$ remains constant. We denote the set of all such points as the stable set $\mathcal{S}$. Its complement, $\mathcal{U} = \mathcal{K} \setminus \mathcal{S}$, is defined as the \textit{unstable collapse set}.
\end{definition}


In practice, the unstable collapse set $\cU$ typically forms a zero-measure set. For example, given a 1D two-branch case with $f_1(x)=0$ and $f_2(x)=\mathrm{ReLU}(x)$, $\cU$ contains only a single point: $\cU=\{0\}$. Only in rare, artificial scenarios does $\cU$ form a substantial subset. For instance, if $f_1(x)=0$ and $f_2(x)=\mathrm{dist}(x,\cC)$ where $\cC$ is a fat Cantor set, then $\cU = \cC$, which has a strictly positive measure.

With these definitions, we are ready to state our main finding. In contrast to the naive construction \eqref{eq:naive}, proving the existence of an operator that is also regular with respect to its input requires a highly non-trivial construction. The theorem is provided below, with the detailed proof in Appendix \ref{app:express}.

\begin{theorem}
\label{thm:expressivity}
Let $\cK\subset \mathbb{R}^d$ be bounded and $\cY=\mathbb{R}^m$. Assume the branch maps $f_1,\dots,f_n:\cK\to \cY$ are locally Lipschitz. Then there exists an operator $g:\cY\times\cK\to\cY$ satisfying:
\begin{enumerate}[leftmargin=15pt]
\item {(Input Regularity).} For every $\vy\in \cY$, the map $\vx\mapsto g(\vy,\vx)$ is globally Lipschitz on $\cK$.
\item {(Reliable Convergence).} For every stable input $\vx\in \cK\setminus\cU$, the iteration $\vy_{t+1}=g(\vy_t,\vx)$ converges to a valid branch in $\cF(\vx)$ for Lebesgue-almost every initialization $\vy_0\in \cY$. (The exceptional initializations form a zero-measure set $\Sigma^\circ(\vx)$).
\item {(Attractor Completeness).} For every $\vx\in \cK\setminus\cU$, varying the initialization completely recovers the target solution set. Formally, $\left\{ \lim_{t\to\infty} g^t(\vy_0,\vx) : \vy_0 \in \cY \setminus \Sigma^\circ(\vx) \right\} = \cF(\vx).$
\end{enumerate}
\end{theorem}


The assumptions of Theorem~\ref{thm:expressivity} are deliberately weak. Each branch is required only to be locally Lipschitz, not globally Lipschitz, and \(\cK\) is assumed bounded but need not be compact or connected. This distinction matters because on bounded non-compact domains, local Lipschitzness does not automatically upgrade to global Lipschitzness. For example, \(x\mapsto \log x\) and \(x\mapsto \sqrt{x}\) on \((0,1]\), \(x\mapsto 1/x\) on \((-1,0)\cup(0,1)\), and \(x\mapsto \tan x\) on \((-\pi/2,\pi/2)\) are locally Lipschitz at every point of their bounded domains, but none is globally Lipschitz because the domain excludes a singular endpoint or interior gap. Thus the theorem covers genuinely non-compact, disconnected, and locally singular branch structures while still guaranteeing a globally input-Lipschitz recurrent operator.


\textbf{Numerical verification.}
As an initial verification of Theorem~\ref{thm:expressivity}, we evaluate two toy multi-solution problems in an idealized setting where all valid labels are known. Crucially, we do not force the model to fit a predefined branch for a given input $x$. Instead, we supply the full set of labels (Algorithm~\ref{alg:toy-supervised}), sample random initializations $y_0$, dynamically match each initialization to its nearest label, and train the weight-tied dynamics to converge accordingly. Details are in Appendix~\ref{app:toy-experiments}.

Figure~\ref{fig:toy-1} provides the numerical evidence. In both toy examples, the iterates of the learned weight-tied dynamics evolve regularly with respect to $(y_0,x)$, and all different initializations are driven toward different valid branches of the target set-valued map. At the same time, the collection of converged states recovers the full branch structure rather than collapsing to a single solution. This is exactly the behavior predicted by Theorem~\ref{thm:expressivity}: a single regular weight-tied dynamics can represent a multi-branch function, and the branch realized at convergence can depend on the initialization while remaining consistent with the target set-valued mapping. Quantitative results are given in Table \ref{tab:alternating-branch}.

\begin{figure}[t]
  \centering
  \includegraphics[width=\linewidth]{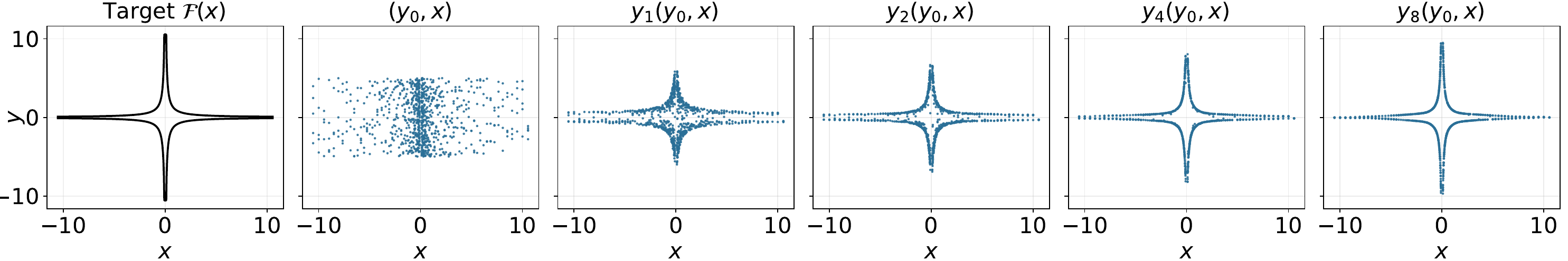}

  
  \includegraphics[width=\linewidth]{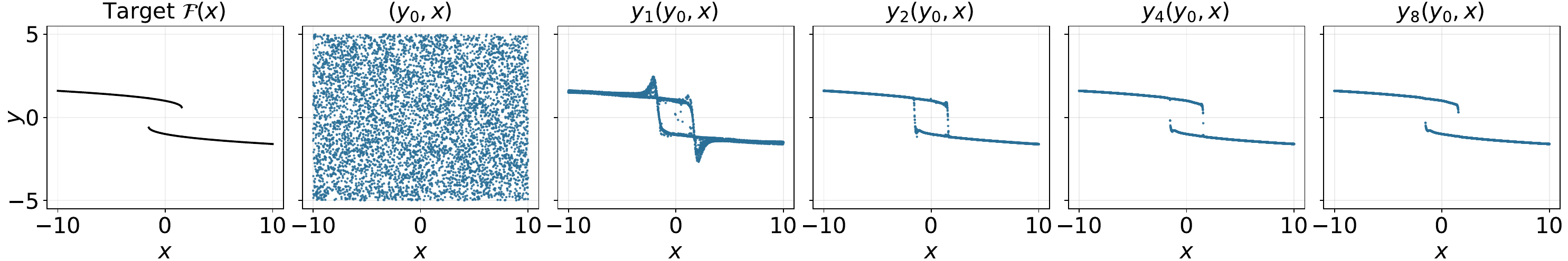}

  \caption{ Numerical verification of Theorem~\ref{thm:expressivity}. Each row shows one toy set-valued map. The first column shows the target branches \(\cF(x)\); the second column shows sampled initial states \(y_0\); and the remaining columns show the iterates \(y_1,y_2,y_4,y_8\) produced by the same weight-tied update map. For each fixed input \(x\), different random initializations move toward different admissible branches. As the iteration proceeds, the point cloud aligns with the full target set.}
  \label{fig:toy-1}
\end{figure}
%

\subsection{When is Multi-Branch Learning Preferable?}
\label{sec:multi-branch-benefit}


Let us now address (Q2): if multiple valid solution branches exist, what is the fundamental theoretical benefit of representing the entire set-valued map rather than isolating a single branch? Returning to the recurrent operator $g$ constructed in Theorem~\ref{thm:expressivity}, we find that multi-branch dynamics provide a natural, built-in regularization. Fixing an initialization $\vy_0$ implicitly defines a solution selector $u(\vy_0, \vx)$ that routes the input to a specific branch. The following theorem guarantees that this dynamically induced selector is exceptionally well-behaved—it is locally Lipschitz almost everywhere.

\begin{theorem}[Lipschitz selector]
\label{thm:main-lip-selector}
Let \(\cK\subset \mathbb{R}^d\) be bounded and measurable, \(\cY=\mathbb{R}^m\), and \(f_1,\ldots,f_n:\cK\to \cY\) be locally Lipschitz branch maps. Let \(\cK^\star := \cK\setminus\cU\) denote the stable domain, and let \(g:\cY\times \cK\to \cY\) be the operator constructed in Theorem \ref{thm:expressivity}. For \(\vy_0\in \cY\), define  $u(\vy_0,\vx) := \lim_{t\to\infty} g^t(\vy_0,\vx)$.
Then, for Lebesgue-almost every initialization \(\vy_0 \in \cY\), it holds that
\[ \Big| \big\{  \vx \in \cK^\star: \vx \mapsto u(\vy_0,\vx) \text{~is not locally Lipschitz at $\vx$} \big\} \Big| = 0 \]

Here, $|\cdot|$ denotes the $d$-dimensional Lebesgue measure. 
\end{theorem}


Theorem~\ref{thm:main-lip-selector} says that, for almost every initialization, the recurrent dynamics does not choose branches in an arbitrary way. It may switch between branches, but the set of inputs where the induced selector fails to be locally Lipschitz has measure zero. This provides a structural contrast with the usual supervised reduction of a multi-solution problem. In that reduction, one first constructs a manual selector \(u:\cK^\star\to\cY\) by choosing \(u(\vx)\in\cF(\vx)\) for every input \(\vx\). Such a choice may come from a solver initialization, an ordering convention, a symmetry-breaking rule, or simply from whichever solution is observed. Unless additional regularity is imposed, there is no reason for this selector to respect the geometry of the branches: nearby inputs can be assigned to far-apart solutions even when all underlying branches are regular. The following theorem formalizes this failure mode.

\begin{theorem}
\label{thm:arbitrary-irregular-selector}
Let $\cK^\star\subset \mathbb{R}^d$ be bounded and measurable, and let $\cY=\mathbb{R}^m$. Let $f_1,\ldots,f_n:\cK^\star\to\cY$ be locally Lipschitz branch maps.
Define the genuinely multi-branch region
\[
    \cM := \Big\{\vx\in\cK^\star:\exists i\neq j
    \text{ such that } f_i(\vx)\neq f_j(\vx)\Big\}.
\]

Then, for every $a\in[0,|\cM|]$, there exists a measurable solution selector $u_a:\cK^\star\to\cY$ such that
\[
   u_a(\vx)\in\cF(\vx), ~~\text{for all }\vx \in \cK^\star, \qquad \text{ and } \big| \operatorname{Disc}_{\cK^\star}(u_a) \big| = a.
\]

Here $\operatorname{Disc}_{\cK^\star}(u_a)$ denotes the set of discontinuity points of $u_a$ relative to the domain $\cK^\star$.
\end{theorem}


Theorem~\ref{thm:arbitrary-irregular-selector} describes a critical danger in standard pipelines: arbitrarily selected branches inject artificial noise into the mapping. If a neural network is forced to fit these discontinuous manual labels, it must learn unphysical, steep decision boundaries that severely harm generalization.

\textit{Importantly, this does not mean multi-branch learning should unconditionally replace conventional supervised learning.} If a smooth selector is known, or if the task explicitly requires convergence to a particular branch, then single-branch supervision is a natural objective. However, when no such selector is available, forcing the model to imitate one arbitrary branch can introduce avoidable jumps.


\textbf{Numerical verification.}
We numerically compare the selector induced by the learned weight-tied dynamics with manually prescribed single-branch selectors on the same two toy examples. We construct a family of artificial single-branch datasets by partitioning the input domain into $N_{\rm int}$ intervals and forcing one fixed branch to be chosen on each interval, with neighboring intervals alternating between the two branches. Thus, as $N_{\rm int}$ increases, the manually selected target becomes increasingly oscillatory and introduces more unnecessary, artificial switching. 

\begin{figure}[t]
  \centering
  \includegraphics[width=\linewidth]{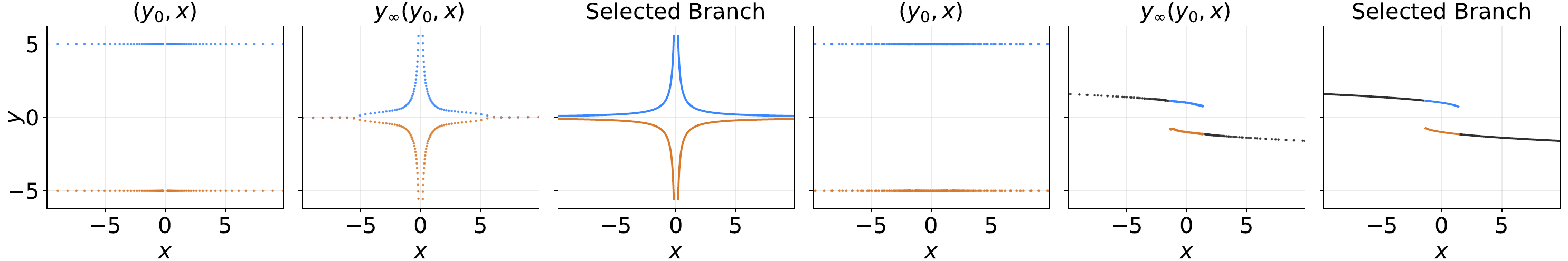}


  \includegraphics[width=\linewidth]{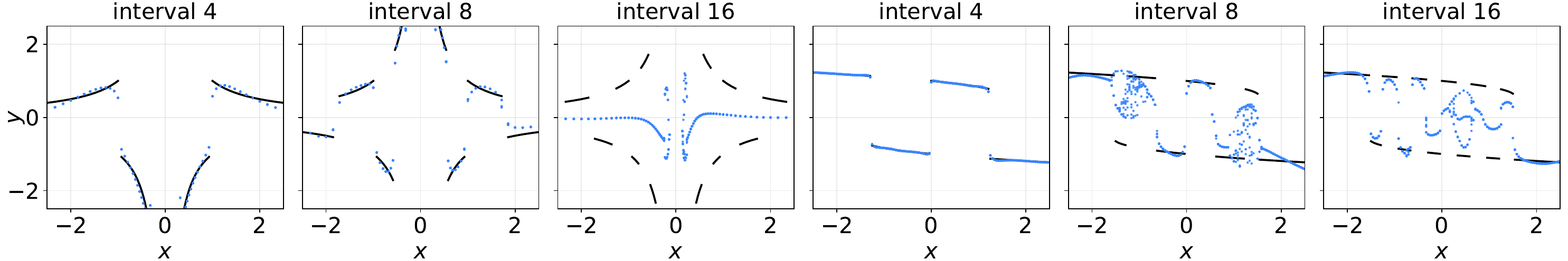}

  \caption{
  Comparison of dynamical (top row) and manual (bottom row) solution selectors for Example 1 (left) and Example 2 (right). The dynamic model naturally maps a fixed initialization $y_0$ to a regular converged state $y_\infty(y_0,x)$. Conversely, manual selection introduces arbitrary switching and artificial discontinuities as the interval increases, hindering training and generalization.
  }
  \label{fig:toy-2}
\end{figure}
\begin{table}[t]
  \begin{center}
    \caption{
    Test error for the bifurcation-model selector and manually selected branch targets.
    The manual selectors alternate between branches over \(N_{\mathrm{int}}\) intervals. $N_{\mathrm{int}} \in \{2,4,8,16,32,64\}$.
    }
    \label{tab:alternating-branch}
    \begin{tabular}{lccccccc}
    \toprule
    & \multirow{2}{*}{Bifurcation}
    & \multicolumn{6}{c}{Manual selector: \(N_{\mathrm{int}}\)} \\
    \cmidrule(lr){3-8}
    Problem & & 2 & 4 & 8 & 16 & 32 & 64 \\
    \midrule
    Toy example 1   & 0.0028 & 0.0181 & 0.1847 & 0.8982 & 5.7052 & 9.6052 & 10.5995 \\
    Toy example 2 & 0.0002 & 0.0159 & 0.0321 & 0.2401 & 0.2698 & 0.5019 & 0.5767 \\
    \bottomrule
    \end{tabular}
  \end{center}
\end{table}

Figure~\ref{fig:toy-2} shows the effect: the selector induced by the implicit dynamics remains smooth and natural, while the human-designed selectors become progressively more irregular. Table~\ref{tab:alternating-branch} shows the same phenomenon quantitatively: as the interval count increases, both training and test errors grow substantially. These experiments validate Theorem~\ref{thm:main-lip-selector} by illustrating that the selector induced by the weight-tied dynamics is regular, and they validate Theorem~\ref{thm:arbitrary-irregular-selector} by showing that arbitrary branch assignment can create artificial discontinuities that make the learning problem significantly harder.

\section{Discovering Multiple Solutions Without Branch Labels}
\label{sec:ising}


The toy experiments in Section~\ref{sec:theory} verify the theory in an idealized setting: for each input \(\vx\), all valid branch labels \((\vy_1,\ldots,\vy_n)\) are known, and each sampled initialization \(\vy_0\) is trained to converge to its nearest branch label. In realistic scientific systems, however, enumerating all solution branches is usually impossible or prohibitively expensive. What is often available instead is an unsupervised objective \(E(\vy,\vx)\), such as an energy functional, whose low values indicate physically valid solutions.

We therefore replace the nearest-branch fitting loss in Algorithm~\ref{alg:toy-supervised} with the unsupervised procedure in Algorithm~\ref{alg:unsupervised}. For each input \(\vx\), we sample multiple initial states \(\vy_0\), unroll the same weight-tied dynamics, and minimize the terminal physical objective \(E(\vy_T,\vx)\). No branch labels are provided, 
and the dynamics are free to converge to any physically valid branch.

This raises several questions. \textit{Can a branch-flexible recurrent model discover multiple solutions for the same input without being shown all branches?} \textit{Does unsupervised branch-flexible training outperform supervised fitting to a single solver-selected label?} \textit{And does avoiding a prescribed branch make the learned solution map less brittle?} We study these questions below in a frustrated Ising system, where many low-energy spin configurations can coexist for the same graph.

\begin{figure}
\centering

\begin{minipage}[t]{0.495\textwidth}
\begin{algorithm}[H]
\caption{Multi-branch fitting}
\label{alg:toy-supervised}
\small
\begin{algorithmic}[1]
\Require Full branch data \(\cD = \left\{\vx_i,\{\vy_{i,j}\}_{j=1}^{n_i}\right\}_i\)
    \For{each mini-batch \(\mathcal B\subset\mathcal D\)}
    \State Sample initialized states \(\vy_{i,k}^{(0)}\sim\rho\) for 
    \[
    \left(\vx_i,\{\vy_{i,j}\}_{j=1}^{n_i}\right)\in\mathcal B, 
    \quad k = 1,\cdots,M
    \]

    \State Choose the closest branch:
    \[
    \vy^\star_{i,k} 
    = 
    \arg\min_{\vy\in\{\vy_{i,1},\ldots,\vy_{i,n_i}\}} 
    \|\vy-\vy^{(0)}_{i,k}\|
    \]

    \State Unroll for \(t=0,1,\cdots,T-1\): 
    \[
    \vy^{(t+1)}_{i,k}
    =
    g_\theta\left(\vy^{(t)}_{i,k},\vx_i\right)
    \]

    \State Minimize $\mathcal{L}(\theta) := \mathrm{mean}\left(\|\vy^{(T)}_{i,k}-\vy^\star_{i,k}\|_2^2\right)$
    \EndFor
\end{algorithmic}
\end{algorithm}
\end{minipage}
\hfill
\begin{minipage}[t]{0.495\textwidth}
\begin{algorithm}[H]
\caption{Multi-branch discovery}
\label{alg:unsupervised}
\small
\begin{algorithmic}[1]
\Require Input data \(\cD = \left\{\vx_i\right\}_i\)
    \For{each mini-batch \(\mathcal B\subset\mathcal D\)}
    \State Sample initialized states \(\vy_{i,k}^{(0)}\sim\rho\) for 
    \[
    \vx_i\in\mathcal B, 
    \quad k = 1,\cdots,M
    \]

    \Statex \phantom{Choose the closest branch:}
    \[
    \phantom{
    \vy^\star_{i,k} 
    = 
    \arg\min_{\vy\in\{\vy_{i,1},\ldots,\vy_{i,n_i}\}} 
    \|\vy-\vy^{(0)}_{i,k}\|
    }
    \]

    \State Unroll for \(t=0,1,\cdots,T-1\): 
    \[
    \vy^{(t+1)}_{i,k}
    =
    g_\theta\left(\vy^{(t)}_{i,k},\vx_i\right)
    \]

    \State Minimize $\mathcal{L}(\theta) := \mathrm{mean}\left(E\left(\vy^{(T)}_{i,k}, \vx_i\right)\right)$
    \EndFor
\end{algorithmic}
\end{algorithm}
\end{minipage}
\end{figure}

\textbf{Basic formulation.} The Ising model \cite{ising1925beitrag} is a fundamental mathematical model in statistical mechanics used to study the magnetic properties of materials. Each Ising instance is a weighted graph $\cG=(\cV,\cE)$ with coupling weights $J_{ij}$ on the edges. A spin assignment is a vector $\vs\in\{-1,+1\}^{|\cV|}$, representing the binary magnetic dipole (up or down) of each node. Finding the system's stable physical ground state at zero temperature corresponds to solving the minimal energy problem:
\begin{equation}
\label{eq:physical-energy}
\min_{\vs\in\{-1,+1\}^{|\cV|}} E_{\rm Ising}(\vs, \cG) = -\frac{1}{2}\sum_{(i,j)\in \cE} J_{ij} s_i s_j.
\end{equation}

When all couplings are positive ($J_{ij}>0$), the ferromagnetic ground states are trivial: all spins align, giving the uniform all-$(+1)$ or all-$(-1)$ states. We therefore focus on the antiferromagnetic case ($J_{ij}<0$), where \textit{geometric frustration} \cite{moessner2006geometrical} makes \eqref{alg:toy-supervised} a degenerate landscape with many valid low-energy configurations, i.e., the map from a graph to its low-energy states is set-valued.

\textbf{Model.}
To parameterize the recurrent update $g$, we use a standard message-passing graph neural network (GNN) \cite{gilmer2017neural} on
the weighted graph $\cG=(\cV,\cE)$. At iteration $t$, the current continuous spin vector $\vs_t$ is
treated as the node feature, and the GNN updates it using the graph structure and edge weights
$J_{ij}$: $\vs_{t+1}=g_{\rm GNN}(\vs_t,\cG)$.
Starting from different initializations $\vs_0$, the same weight-tied update is unrolled for $T$ steps,
and the final prediction is projected to the spin range by $\hat{\vs}=\vs_T$.
The goal of training is that, as $T$ grows, the iterates converge to low-energy spin configurations of
the same graph; different initializations may therefore lead to different low-energy states. 

As a non-recurrent counterpart, we also consider a vanilla GNN baseline. This model applies message passing to the weighted graph as a feed-forward predictor and directly outputs a spin $\hat{\vs}=h_{\rm GNN}(\cG)$.

\textbf{Loss and evaluation.}During neural training, the model outputs continuous spins $\hat \vs\in[-1,1]^{|\cV|}$. To penalizes non-binary outputs, we use the following unsupervised training loss:
\begin{equation}
\label{eq:unsuper-loss}
E(\hat \vs;\cG) = \frac{1}{|\cV|}\Big(-\frac{1}{2}\sum_{(i,j)\in \cE} J_{ij} \hat s_i \hat s_j + \sum_{i\in \cV} (\hat s_i^2-1)^2\Big).
\end{equation}

Dataset and other experiment details are in Appendix \ref{app:ising-scaling}.

\textbf{Can unsupervised training discover many solutions?}
Yes. In this Ising setting, the energy-trained weight-tied dynamics shows clear multi-solution behavior. Figure~\ref{fig:ising-solutions} gives a representative visual example, and Table~\ref{tab:ising-accuracy} provides a quantitative summary. The energy-trained recurrent GNN finds many distinct solutions on average: its trajectories bifurcate over the iterates. A fuller histogram of the solution counts is provided in Appendix~\ref{app:ising-scaling}.

\begin{figure}[t]
    \centering
    \includegraphics[width=\linewidth]{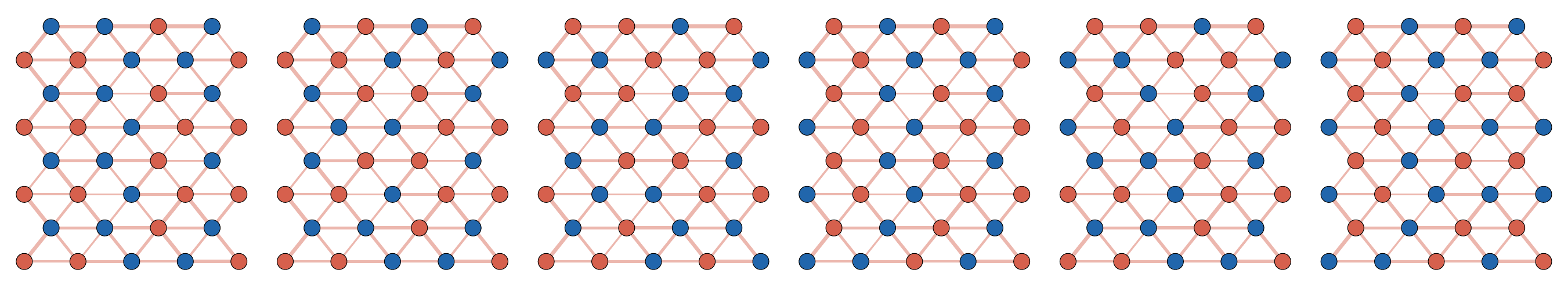}

    \caption{Unsupervised multi-solution discovery. The same Ising graph instance yields distinct solutions when starting from different initializations. Red nodes denote $+1$ spins and blue for $-1$.}
    \label{fig:ising-solutions}
\end{figure}

\textbf{Label fitting or energy minimization?}
We next compare two training objectives: fitting one solver-selected label versus minimizing the physical energy directly.  
For the label-fitting objective, we use Gurobi \cite{gurobi} to solve the binary quadratic Ising problem in~\eqref{eq:physical-energy}, and train the model to imitate one returned spin configuration for each graph.  
For the energy objective, we train directly with \eqref{eq:unsuper-loss}.
Table~\ref{tab:ising-accuracy} shows that energy minimization gives lower rounded-spin energy for both vanilla and recurrent GNNs. 
The vanilla GNN still outputs a single solution, but the energy objective lets it choose a low-energy branch instead of imitating an arbitrary Gurobi-selected one. 
The recurrent GNN has additional freedom: different initializations can follow different trajectories and settle into different low-energy basins. 
Thus, energy-based training not only avoids a potentially brittle single-branch label, but also enables the recurrent model to exploit its multi-solution dynamics.

  \begin{table}[t]
  \centering
    \caption{Quantitative summary of the Ising experiments. ``Train energy'' and ``test energy'' denote the physical Ising energy $E_{\rm Ising}(\bar{\vs})/|\cV|$ evaluated after rounding the predicted spins $\bar{\vs}=\mathrm{sign}(\hat\vs)$, on the training and test sets respectively. The number of solutions is computed on 200 test graphs using 20 trajectories per graph. All values are reported as mean $\pm$ standard deviation.}
    \label{tab:ising-accuracy}
  \begin{tabular}{llccc}
  \toprule
Objective &  Model & Train energy & Test energy & \# Solution (Out of 20) \\
  \midrule
Gurobi label &  Vanilla GNN & -0.478 $\pm$ 0.327 & -0.059 $\pm$ 0.358 & 1.000 $\pm$ 0.000\\
Gurobi label &  Recurrent GNN & -0.513 $\pm$ 0.315 & -0.092 $\pm$ 0.361 & 1.000 $\pm$ 0.000 \\\midrule
Energy-based &  Vanilla GNN & -0.768 $\pm$ 0.181 & -0.750 $\pm$ 0.190 & 1.000 $\pm$ 0.000 \\
Energy-based &  Recurrent GNN & \textbf{-0.979 $\pm$ 0.101} & \textbf{-0.983 $\pm$ 0.101} & 13.635 $\pm$ 5.760 \\
  \bottomrule
  \end{tabular}
  
  \vspace{5mm}

\caption{Empirical perturbation sensitivity at noise level $\varepsilon=0.005$. Values represent percentiles of the normalized output variation, $\|\hat{\vs}(J+\Delta J)-\hat{\vs}(J)\|_2/\|\Delta J\|_2$, across the test set. To illustrate, $p_{40} = 15.184$ indicates that 40\% of the tested perturbations resulted in an output variation of 15.184 or lower. Energy-based models remain highly stable for the vast majority of inputs, whereas models trained on Gurobi labels exhibit sensitivity globally, which supports our theory in Section \ref{sec:multi-branch-benefit}.}
\label{tab:ising-lipschitz}
\begin{tabular}{llcccccccc}
\toprule
Obj. & Model & $p_{40}$ & $p_{50}$ & $p_{60}$ & $p_{70}$ & $p_{80}$ & $p_{90}$ & $p_{95}$ \\\midrule
Gurobi & Vanilla GNN  & 15.184 & 16.730 & 18.411 & 20.693 & 23.767 & 27.601 & 31.606 \\
Gurobi & Recurrent GNN  & 9.919 & 10.762 & 11.600 & 12.591 & 14.020 & 15.612 & 17.523 \\\midrule
Energy & Vanilla GNN   & 0.000 & 0.000 & 0.000 & 0.000 & 0.000 & 0.000 & 25.978  \\
Energy & Recurrent GNN   & 0.000 & 0.000 & 0.000 & 0.000 & 0.000 & 0.025 & 38.168  \\\bottomrule
 \end{tabular}
\end{table}

\textbf{Why does energy minimization yield a more robust solver?}
This gap reflects the regularity viewpoint in Section~\ref{sec:multi-branch-benefit}. 
Under energy minimization, the model is not forced to imitate a solver-selected branch; it can dynamically route each input to a low-energy attraction basin. Theorem~\ref{thm:main-lip-selector} shows that such induced selectors are locally Lipschitz almost everywhere, whereas manual selectors can be arbitrarily irregular (Theorem~\ref{thm:arbitrary-irregular-selector}). Thus, single-label fitting may introduce artificial branch switches.
Table~\ref{tab:ising-lipschitz} supports this contrast empirically. Energy-based  models have near-zero perturbation sensitivity for most inputs, while Gurobi-label models remain sensitive across the distribution. This suggests that energy minimization avoids fitting an unnecessarily irregular selector.

\textbf{Larger graphs.} Appendix~\ref{app:ising-scaling} extends the evaluation to graphs with up to $\sim$15,000 nodes.

\section{Encouraging Bifurcations: Explicit Diversity in Training }
\label{sec:allen-cahn}


The Ising experiment shows a favorable regime where unsupervised weight-tied dynamics naturally uses different initializations to discover multiple low-energy solutions. We now ask a complementary question: \textit{is such diversity automatic once the model has multi-equilibrium capacity?} To answer this question, we use the Allen--Cahn equation as a testbed.

\textbf{Problem formulation.} 
Allen--Cahn equation \cite{allen1979microscopic} is a fundamental reaction-diffusion equation widely used to study complex pattern formation and phase transitions in continuous media.  On the domain $\Omega=(0,2\pi)^2$, with phase field $u:\Omega\to\mathbb R$, forcing $f:\Omega\to\mathbb R$, and $\varepsilon>0$, we consider
\[-\varepsilon^2 \Delta u + u^3 - u - f = 0 \qquad \text{in } \Omega,\]

with periodic boundary conditions, i.e., on the flat torus $\mathbb T^2$. It is the Euler--Lagrange equation of
\[E_{\rm AC}(u;f) = \frac{1}{|\Omega|} \int_{\Omega} \left[ \frac{\varepsilon^2}{2}\|\nabla u(x)\|_2^2 + \frac{1}{4}\bigl(u(x)^2-1\bigr)^2 - f(x)u(x) \right]\,\mathrm{d}x.\]

The energy balances interface smoothing, external forcing, and a double-well potential that drives the system toward distinct macroscopic phases ($u \approx \pm 1$). This competition naturally creates multiple stable steady states for the exact same forcing $f$ \cite{andrade2026bubbles}. This continuous, multi-basin landscape provides an ideal testbed, requiring the network to genuinely capture coexisting attraction basins.

\textbf{Model.} Given a forcing field \(f\), our goal is to learn a solver that maps \(f\) to one or more steady phase fields \(u\). Since both the input and output are periodic functions on \(\Omega\), we parameterize the update map with a Fourier neural operator (FNO)~\cite{li2021fourier}. 
We do not use the FNO as a one-shot predictor \(f \mapsto u\). Instead, following the implicit neural operator viewpoint of~\cite{marwah2023deep}, we use it as a shared recurrent update: $u_{t+1} = g_{\rm FNO}(u_t, f)$.
Starting from a randomly generated initial state $u_0$, we unroll the weight-tied dynamics for $T$ iterations. During training, the model parameters are optimized entirely unsupervised by directly minimizing the terminal physical energy $E_{\rm AC}(u_T; f)$.

\textbf{Evaluation metrics.} To evaluate physical accuracy, we measure the steady-state PDE residual using the discrete mean squared error (MSE) over an $N \times N$ spatial grid:
\[\mathrm{MSE}_{\mathrm{res}}(u;f) = \frac{1}{N^2} \sum_{p=1}^{N}\sum_{q=1}^{N} \left(-\varepsilon^2 \Delta u + u^3 - u - f\right)_{p,q}^2.\]

\textbf{Does multi-equilibrium capacity imply automatic branch discovery?} Not necessarily.
Figure~\ref{fig:eps001-lambda0-vs-p3} shows that the energy-only model maps very different initial states to visually similar final states, indicating collapse to a dominant branch. 
Table~\ref{tab:fixedeps_eps001_implicit_5model} confirms this quantitatively: the energy-only model has only one discovered solution. 
Thus, while the weight-tied dynamics has the capacity to represent multiple attractors, minimizing the energy alone does not force those attractors to be used.

\textbf{How does energy training compare with IMEX-label fitting?}
Although the energy-only model does not automatically bifurcate, it is still much more accurate
than fitting a single IMEX-generated label; see
Table~\ref{tab:fixedeps_eps001_implicit_5model}. For the single-label target, we use the
classical periodic IMEX solver~\cite{chen1998applications} to construct one target branch for
each forcing field. Forcing the network to fit this selected branch gives substantially worse
physical residuals and energy than the energy-only model. By contrast, energy training does not
prescribe which branch the network must imitate; it only asks the dynamics to reach a low-energy
steady state. The learned model can therefore settle on the branch that is easiest to represent and
optimize, rather than spending capacity matching a particular label-selection rule.

\textbf{If we still want multiple solutions, how do we achieve that?}
We add an explicit diversity-promoting term to the energy objective. For a fixed forcing field \(f\),
we sample \(M\) independent initial states \(u_k^{(0)}\sim \rho\), evolve each of them $k=1,\ldots,M$ with the same
weight-tied dynamics, $u_k^{(t+1)} = g_\theta(u_k^{(t)},f)$.
Denote the final states by \(\hat u_k := u_k^{(T)}\). The diversity-regularized objective is
\[
  \mathcal{L}_{\lambda}(\theta;f)
  =
  \frac{1}{M}\sum_{k=1}^{M} E_{\rm AC}(\hat u_k;f)
  -
  \lambda\,
  \mathrm{Diversity}\!\left(\{\hat u_k\}_{k=1}^{M}\right),
\]

where
\[
  \mathrm{Diversity}\!\left(\{\hat u_k\}_{k=1}^{M}\right)
  =
  \frac{1}{M(M-1)}
  \sum_{i\neq j}
  \frac{1}{N^2}\|\hat u_i-\hat u_j\|_2^2 .
\]

Here, the penalty weight $\lambda$ is decayed over the progress of training. For clarity, we refer to each specific decay schedule by its initial penalty weight, $\lambda_{\rm init}$ (full details in Appendix \ref{app:ac-additional}).

Figure~\ref{fig:eps001-lambda0-vs-p3} illustrates the effect. The top row shows five independent initial states for the same forcing field. With energy-only training, \(\lambda_{\rm init}=0\), the middle row shows that these trajectories collapse to visually similar final states. With diversity regularization, \(\lambda_{\rm init}=1.2\), the bottom row shows that the same initial states lead to visibly distinct steady-state candidates. Table~\ref{tab:fixedeps_eps001_implicit_5model} confirms the same trend: the average number of detected solution clusters increases from \(1.00\) at \(\lambda_{\rm init}=0\) to \(16.11\) at \(\lambda_{\rm init}=1.2\).

\begin{figure}[t]
  \centering
  \includegraphics[width=\linewidth]{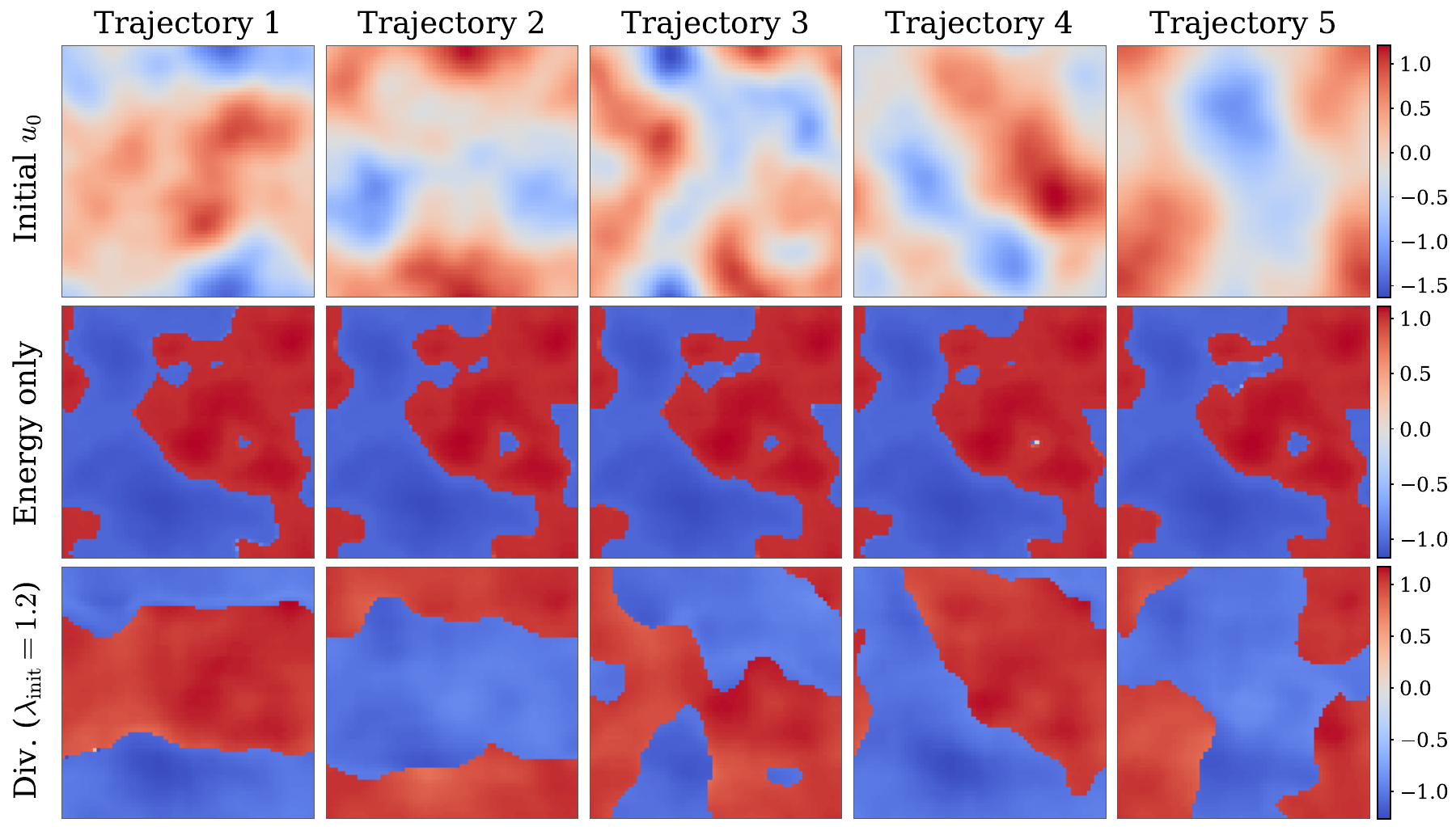}

  \caption{Five trajectories for the same \(f\). \textbf{Top row:} five independent initial states.
\textbf{Middle row:} final predictions from the energy-only model: collapse to visually similar states. 
\textbf{Bottom row:} final predictions from the diversity-regularized model: discover distinct steady-state candidates.}
  \label{fig:eps001-lambda0-vs-p3}
\end{figure}

\textbf{Is there an accuracy--diversity tradeoff?}
Yes. Table~\ref{tab:fixedeps_eps001_implicit_5model} shows a clear tradeoff. The energy-only model gives the best physical accuracy, with the lowest residual error and energy, but it discovers only one solution cluster. Increasing the initial diversity weight \(\lambda_{\rm init}\) produces more distinct outputs: for example, \(\lambda_{\rm init}=1.2\) increases the average number of solution clusters from \(1.00\) to \(16.11\). However, this comes at a cost: the test residual error increases from \(9.19\times 10^{-4}\) to \(5.74\times 10^{-3}\), and the test energy becomes less favorable.
Thus, in this experiment, multi-solution discovery is not free. 
This tradeoff gives a practical way to tune the solver: setting $\lambda=0$ prioritizes physical accuracy, while increasing $\lambda$ encourages broader exploration of the solution landscape. 

\begin{table}[t]
\centering
\caption{Full-dataset Allen--Cahn evaluation. For each forcing instance, we run 20 trajectories for 50
iterations. Train/Test error denote the residual MSE \(\mathrm{MSE}_{\mathrm{res}}(u;f)\), and
Train/Test energy denote the energy \(E_{\rm AC}(u;f)\). Lower error and lower
energy indicate better physical accuracy; residual MSE is nonnegative, while the energy may be
negative. ``\# Sols'' is the average number of distinct output clusters among the final states,
so larger values indicate greater solution diversity.}
\label{tab:fixedeps_eps001_implicit_5model}
\begin{tabular}{lrrrrr}
\toprule
Training obj. & Train error & Test error & Train energy & Test energy & \# Sols \\
\midrule
IMEX label
& $9.58 \times 10^{-2}$ & $9.55 \times 10^{-2}$
& $1.43 \times 10^{-1}$ & $1.43 \times 10^{-1}$ & $1.00$ \\
\midrule
Energy only
& $\mathbf{9.08 \times 10^{-4}}$ & $\mathbf{9.19 \times 10^{-4}}$
& $\mathbf{-1.15 \times 10^{-1}}$ & $\mathbf{-1.15 \times 10^{-1}}$ & $1.00$ \\\midrule
Div. $\lambda_{\rm init}=0.5$
& $1.54 \times 10^{-3}$ & $1.55 \times 10^{-3}$
& $\mathbf{-1.15 \times 10^{-1}}$ & $-1.14 \times 10^{-1}$ & $1.00$ \\
Div. $\lambda_{\rm init}=1.2$
& $5.77 \times 10^{-3}$ & $5.74 \times 10^{-3}$
& $-4.29 \times 10^{-2}$ & $-4.22 \times 10^{-2}$ & $16.11$ \\
Div. $\lambda_{\rm init}=2.0$
& $4.06 \times 10^{-3}$ & $4.07 \times 10^{-3}$
& $-8.11 \times 10^{-3}$ & $-7.95 \times 10^{-3}$ & ${18.84}$ \\
\bottomrule
\end{tabular}
\end{table}

\textbf{Additional experiments.}
Appendix~\ref{app:ac-additional} gives two further studies: learned warm starts improve IMEX solver, and the same accuracy--diversity tradeoff persists beyond the \(64\times64\) setting.

\section{Conclusions}

We studied learning problems where one input admits multiple valid solutions. Bifurcation models represent this multiplicity through weight-tied dynamics, whose different initializations can converge to different stable equilibria. We proved that broad set-valued maps can be represented by regular dynamics and that the induced selectors are almost everywhere regular, unlike arbitrary manual branch labels. Empirically, these dynamics can naturally discover many low-energy Ising solutions without branch labels. In Allen--Cahn, however, energy minimization alone may collapse to a dominant branch; explicitly promoting diversity recovers multiple branches, revealing an accuracy--diversity tradeoff. This makes diversity a tunable modeling objective: the training objective can bias the solver toward either higher physical accuracy or broader exploration of the solution set.

\newpage
\bibliographystyle{amsxport}
\bibliography{references}


\newpage 

\appendix

\section{Proof of Theorem \ref{thm:expressivity}}
\label{app:express}


Throughout this section, let \(\cK\subset\mathbb R^d\) be bounded, let \(\cY=\mathbb R^m\), and let
\[
f_1,\dots,f_n:\cK\to \cY
\]
be locally Lipschitz branch maps. These branches may collapse. Define the finite set of distinct branch values
\[
\cF(x):=\{f_1(\vx),\dots,f_n(\vx)\}\subset \cY.
\]

\subsection{Nearest solution selector}

\begin{definition}[Priority nearest selector]
\label{def:nearest-selector}
Fix a priority order
\[
1\prec 2\prec\cdots\prec n.
\]
For \(\vy\in \cY\) and \(\vx\in \cK\), define
\[
\pi_\vy(\vx) := \min \operatorname*{argmin}_{1\le i\le n}\|\vy-f_i(\vx)\|,
\]
where the minimum is taken according to the fixed priority order. Define
\[
P_\vy(\vx):=f_{\pi_\vy(\vx)}(\vx).
\]
Thus \(P_\vy(\vx)\) is single-valued for every \((\vy,\vx)\in \cY\times \cK\).
\end{definition}

\begin{definition}[Voronoi switching set]
Define the reduced switching set
\[
\Sigma^\circ(\vx) := \Bigl\{ \vy\in \cY: \exists i\ne j,\ f_i(\vx)\ne f_j(\vx),\ \|\vy-f_i(\vx)\|=\|\vy-f_j(\vx)\| =\min_k\|\vy-f_k(\vx)\| \Bigr\}.
\]
Define the graph-closure switching set
\[
\widetilde\Sigma(\vx_0) := \left\{
\vy\in \cY: \exists \vx_\ell\to \vx_0,\ \exists \vy_\ell\to \vy \text{ such that } \vy_\ell\in\Sigma^\circ(\vx_\ell)
\right\}.
\]
Equivalently,
\[
\operatorname{graph}(\widetilde\Sigma)
=
\overline{\operatorname{graph}(\Sigma^\circ)}
\subset \cK\times \cY.
\]
In particular, \(\operatorname{graph}(\widetilde\Sigma)\) is closed.
\end{definition}

\begin{lemma}[Continuity of solution selector]
\label{lemma:lip-p-operator}
Let \(K\subset \mathbb{R}^d\) and let \(Y=\mathbb{R}^m\) with the Euclidean norm. Let $f_1,\dots,f_n:\cK\to \cY$ be locally Lipschitz. For fixed \(\vy\in \cY\), define
\[
\cJ_\vy:=\{\vx\in \cK:(\vx,\vy)\in\operatorname{graph}(\widetilde\Sigma)\}.
\]
Then \(P_\vy\) is locally Lipschitz on $\cK\setminus \cJ_\vy$.
\end{lemma}

\begin{proof}
Fix \(\vx_0\in \cK\setminus \cJ_\vy\). Since
\[
(\vx_0,\vy)\notin \operatorname{graph}(\widetilde\Sigma)
\]
and \(\operatorname{graph}(\widetilde\Sigma)\) is closed, there exist a relative neighborhood
\[
\sV\subset \cK
\]
of \(\vx_0\), and a number \(\varepsilon>0\), such that
\[
(\sV\times \sB_\varepsilon(\vy))
\cap \operatorname{graph}(\Sigma^\circ)
=\varnothing.
\]
Equivalently, for every \(\vx\in \sV\) and every \(\vq\in \sB_\varepsilon(\vy)\), the point \(\vq\) is not on the reduced Voronoi switching set of \(\cF(\vx)\).

For fixed \(\vx\in \sV\), the ball \(\sB_\varepsilon(\vy)\) is connected. Since no point in this ball crosses a reduced switching boundary, the same distinct branch value remains the unique nearest value throughout the ball. Therefore, if
\[
\vu=P_\vy(\vx),
\]
then for every \(\vq\in \sB_\varepsilon(\vy)\) and every \(\vw\in \cF(\vx)\),
\[
\|\vq-\vu\|\le \|\vq-\vw\|.
\]
Moreover, if \(\vw\ne \vu\), choose
\[
\vq = \vy+\frac{\varepsilon}{2}\frac{\vw-\vu}{\|\vw-\vu\|}.
\]
and plug it into $\|\vq-\vu\|^2\le \|\vq-\vw\|^2$ gives
\[
\begin{aligned}
\left\|\vy-\vu+\frac{\varepsilon}{2}\frac{\vw-\vu}{\|\vw-\vu\|}  \right\|^2 
\le & \left\|\vy-\vw+\frac{\varepsilon}{2}\frac{\vw-\vu}{\|\vw-\vu\|} \right\|^2 \\
\|\vy-\vu\|^2 + \varepsilon \left\langle  \vy-\vu, \frac{\vw-\vu}{\|\vw-\vu\|}  \right\rangle + \frac{\varepsilon^2}{4} \leq & \|\vy-\vw\|^2 + \varepsilon \left\langle  \vy-\vw, \frac{\vw-\vu}{\|\vw-\vu\|} \right\rangle + \frac{\varepsilon^2}{4} 
\end{aligned}
\]
and hence by possibly shrinking \(\varepsilon\), we have
\[
\|\vy-\vw\|^2-\|\vy-\vu\|^2 \ge \varepsilon\|\vw-\vu\|
\]
for all \(\vx\in \sV\), \(\vu=P_\vy(\vx)\), and \(\vw\in \cF(\vx)\).

Now shrink \(\sV\) if necessary so that all \(f_i\) are \(L\)-Lipschitz on \(\sV\). Also, let
\[
B:=\sup_{\vx\in \sV,\,1\le i\le n}\|\vy-f_i(\vx)\|<\infty.
\]

Take arbitrary \(\vx,\vz\in \sV\). Let
\[
\vu=P_\vy(\vx),\qquad \vv=P_\vy(\vz).
\]
Choose labels \(i,j\) such that
\[
\vu=f_i(\vx),\qquad \vv=f_j(\vz).
\]
Set
\[
\vr:=f_i(\vz),\qquad \vw:=f_j(\vx).
\]
By \(L\)-Lipschitzness,
\[
\|\vu-\vr\|\le L\|\vx-\vz\|,
\qquad
\|\vv-\vw\|\le L\|\vx-\vz\|.
\]

Apply the margin estimate at \(\vx\) with competitor \(\vw\), and at \(\vz\) with competitor \(\vr\):
\[
\begin{aligned}
\|\vy-\vw\|^2-\|\vy-\vu\|^2 \ge& \varepsilon\|\vw-\vu\|, \\
\|\vy-\vr\|^2-\|\vy-\vv\|^2 \ge& \varepsilon\|\vr-\vv\|.
\end{aligned}
\]
Adding,
\[
\varepsilon(\|\vw-\vu\|+\|\vr-\vv\|) \le \bigl(\|\vy-\vw\|^2-\|\vy-\vv\|^2\bigr) + \bigl(\|\vy-\vr\|^2-\|\vy-\vu\|^2\bigr).
\]
The right-hand side is bounded by $4BL\|\vx-\vz\|$ because
\[
\begin{aligned}
\|\vy-\vw\|^2-\|\vy-\vv\|^2 = & \langle (\vy-\vw) + (\vy-\vv) , \vw - \vv \rangle \\
 \leq & \Big( \|\vy-\vw\| + \|\vy-\vv\| \Big) \|\vw - \vv\| \leq 2B \|\vw - \vv\| \leq 2BL\|\vx-\vz\|,
\end{aligned}
\]
and similarly for $\bigl(\|\vy-\vr\|^2-\|\vy-\vu\|^2\bigr)$.

On the other hand,
\[
\|\vw-\vu\|\ge \|\vu-\vv\|-L\|\vx-\vz\|,
\]
and
\[
\|\vr-\vv\|\ge \|\vu-\vv\|-L\|\vx-\vz\|.
\]
Hence
\[
\varepsilon(2\|\vu-\vv\|-2L\|\vx-\vz\|)
\le
4BL\|\vx-\vz\|.
\]
Therefore
\[
\|\vu-\vv\| \le L\left(1+\frac{2B}{\varepsilon}\right)\|\vx-\vz\|.
\]
Since \(\vu=P_\vy(\vx)\) and \(\vv=P_\vy(\vz)\), this proves the local Lipschitz of \(P_\vy\) at $\vx_0$. Thus \(P_\vy\) is locally Lipschitz on \(\cK\setminus \cJ_\vy\), which concludes the proof.
\end{proof}

\subsection{Bad sets}

Here we restate Definition \ref{def:main-stable-collapse} and provide a more detailed version below.

\begin{definition}[Stable and unstable collapse pattern]
\label{def:app-stable-collapse}
Let \(f_1,\dots,f_n:\cK\to \cY\) be branch maps, and fix \(\vx_0\in \cK\). We say that the labeled branches have a stable collapse pattern near \(\vx_0\) if there exists a relative neighborhood \(\sU\subset \cK\) of \(\vx_0\) and a partition
\[
\{1,\dots,n\}=C_1\cup\cdots\cup C_q
\]
such that, for every \(\vx\in \sU\),
\[
f_i(\vx)=f_j(\vx) \quad\Longleftrightarrow\quad i,j\in C_\alpha
\text{ for some }\alpha.
\]
Thus labels in the same cluster remain collapsed throughout \(\sU\), while labels
in different clusters remain distinct throughout \(\sU\). Let
\[
\mathcal S := \{\vx\in \cK:\text{the branches have a stable collapse pattern near }\vx\}.
\]
Define the unstable collapse set
\[
\mathcal U:=\cK\setminus\mathcal S.
\]
\end{definition}

\begin{lemma}[Closedness of the unstable collapse set]
\label{lemma:closed-1}
The set \(\mathcal U\) is relatively closed in \(\cK\).
\end{lemma}

\begin{proof}
The stable set \(\mathcal S\) is relatively open. Indeed, if a stable collapse partition holds on a relative neighborhood \(\sU\) of \(\vx_0\), then the same partition holds on a smaller relative neighborhood of every point in \(\sU\). Therefore $\mathcal U=\cK\setminus\mathcal S$ is relatively closed.
\end{proof}

\begin{definition}[Switching set.] For each fixed \(\vy\in \cY\), define the switching-in-\(\vx\) set
\[
\mathcal J_\vy := \{\vx\in \cK:\vy\in\widetilde\Sigma(\vx)\}.
\]
\end{definition}

\begin{lemma}[Closedness of switching set]
\label{lemma:closed-2}
For each \(\vy\in \cY\), the set \(\mathcal J_\vy\) is relatively closed in \(\cK\).
\end{lemma}

\begin{proof}
Since \(\operatorname{graph}(\widetilde\Sigma)\) is closed in \(\cK\times \cY\), its section $\cJ_\vy$ is relatively closed in \(\cK\). 
\end{proof}

\begin{definition}[Extended domain and discontinuity points]
First we extend $\cF$ to $\bar \vx \in \overline{\cK} \backslash \cK$ by the limit relative to $\cK$. For each branch $i=1,2,\cdots,n$, we define
\[
f_i(\bar \vx )=\left\{
\begin{aligned}
    & \lim_{\cK\ni \vx\to\bar \vx}f_i(\vx),  \quad &&  \textup{if $\lim_{\cK\ni \vx\to\bar \vx}f_i(\vx)$ exists}, \\
& 0,  && \textup{otherwise}.
\end{aligned}\right.
\]
In addition, the nearest branch selector $P_\vy(\vx)$ can also be extended to $\overline{\cK}$ accordingly. 
\end{definition}

Note that even if $f_i$ is continuously extendable to $\bar \vx$, it is still possible that $f_i$ is not locally Lipschitz continuous at the point $\bar \vx$. An example is $\sqrt{x}$ when it is extended from $(0,1]$ to $0$. We collect all these points (where $f_i$ is not locally Lipschitz) into the set $\cD$:
\[ \cD:= \Big\{ \vx \in \overline{\cK}: \textup{$f_i$ is not locally Lipschitz continuous at $\vx$ for some $i \in \{1,2,\cdots,n\}$.} \Big\} \]

\begin{lemma}[Closedness of discontinuity points] 
\label{lemma:closed-3}
$\cD$ is closed in $\overline{\cK}$ and $\cD \subset \overline{\cK}\backslash\cK$.
\end{lemma}

\begin{proof}
The conclusion $\cD \subset \overline{\cK}\backslash\cK$ directly follows from the definition of $\cD$ and the local Lipschitz property of $f_i$. To show its closedness, we define:
\[ \cD_i:= \Big\{ \vx \in \overline{\cK}: \textup{$f_i$ is not locally Lipschitz continuous at $\vx$.} \Big\}\]
and hence $\cD = \cup_{i=1}^n \cD_i$. Thanks to \cite[Lemma A.2]{liu2026expressive}, each $\cD_i$ is closed. The union of finitely many closed sets is closed, then $\cD$ is closed.
\end{proof}

\begin{definition}[Bad set]
\label{def:joint-bad-set}
Equip \(\cY\times\overline{\cK}\) with the product Euclidean norm. Define $\overline{\cU}^{\,\overline{\cK}}$ as the closure of \(\cU\) relative to \(\overline{\cK}\). The joint bad set is
\[
\cB
:=
\cY\times
\Bigl(\cD\cup \overline{\cU}^{\,\overline{\cK}}\Bigr)
\;\cup\;
\Bigl\{
(\vy,\vx)\in \cY\times\overline{\cK}:
\vy\in\widetilde\Sigma(\vx)
\Bigr\}.
\]
\end{definition}

\begin{lemma}[Closedness of the joint bad set]
The set \(\mathcal B\) is closed in \(\mathcal Y\times\overline{\mathcal K}\).
\end{lemma}

\begin{proof}
By Lemma~\ref{lemma:closed-3}, \(\mathcal D\) is closed in
\(\overline{\mathcal K}\). By definition,
\(\overline{\mathcal U}^{\,\overline{\mathcal K}}\) is closed in
\(\overline{\mathcal K}\). Hence
\[
\mathcal Y\times\bigl(\mathcal D\cup
\overline{\mathcal U}^{\,\overline{\mathcal K}}\bigr)
\]
is closed in \(\mathcal Y\times\overline{\mathcal K}\).

Moreover, \(\operatorname{graph}(\widetilde\Sigma)\) is closed in
\(\overline{\mathcal K}\times\mathcal Y\). Therefore its coordinate-swapped
version
\[
\{(\vy,\vx)\in \mathcal Y\times\overline{\mathcal K}:
\vy\in\widetilde\Sigma(\vx)\}
\]
is closed in \(\mathcal Y\times\overline{\mathcal K}\). Thus
\(\mathcal B\), being the union of two closed sets, is closed.
\end{proof}

Based on the bad set, we can provide a joint-space analogue of Lemma~\ref{lemma:lip-p-operator}.
\begin{lemma}[Joint local Lipschitzness of the nearest selector]
\label{lemma:lip-p-operator-joint}
Define
\[
P(\vy,\vx):=P_\vy(\vx).
\]
Then \(P\) is locally Lipschitz on
\[
(\cY\times\overline{\cK})\setminus \cB .
\]
\end{lemma}

\begin{proof}
Fix \((\vy_0,\vx_0)\notin \cB\). Then
\[
\vx_0\notin \cD\cup \overline{\cU}^{\,\overline{\cK}},
\qquad
\vy_0\notin \widetilde\Sigma(\vx_0).
\]
Therefore the nearest distinct branch value to \(\vy_0\) at \(\vx_0\) is unique.
Let \(C_*\) be the corresponding stable collapse cluster, and choose its
priority representative \(i_*\).

Since \(\vx_0\notin \overline{\cU}^{\,\overline{\cK}}\), after shrinking a
relative neighborhood \(V\subset\overline{\cK}\) of \(\vx_0\), the collapse
pattern of the branches is stable on \(V\cap\cK\). Choose one representative
\(i_\alpha\) for each stable collapse cluster \(C_\alpha\). Since \(C_*\) is
the unique nearest cluster at \((\vy_0,\vx_0)\), there is a positive margin
\[
\gamma
:=
\min_{\alpha\neq *}
\left(
\|\vy_0-f_{i_\alpha}(\vx_0)\|
-
\|\vy_0-f_{i_*}(\vx_0)\|
\right)
>0 .
\]
By continuity of the branch maps and of the functions
\[
(\vy,\vx)\mapsto \|\vy-f_{i_\alpha}(\vx)\|,
\]
after further shrinking \(V\) and taking \(\epsilon>0\) sufficiently small, we
have
\[
\|\vy-f_{i_*}(\vx)\|
<
\|\vy-f_{i_\alpha}(\vx)\|
\qquad
\text{for every }\alpha\neq *,
\]
for all \((\vy,\vx)\in B_\epsilon(\vy_0)\times V\). Hence the same stable
collapse cluster \(C_*\) remains the unique nearest cluster throughout this
product neighborhood. Therefore,
\[
P(\vy,\vx)=f_{i_*}(\vx)
\qquad
\text{for all }(\vy,\vx)\in B_\epsilon(\vy_0)\times V .
\]
Since \(\vx_0\notin \cD\), the branch \(f_i\) is locally Lipschitz near
\(\vx_0\). Therefore, for all \((\vy,\vx),(\vy',\vx')\in B_\epsilon(\vy_0)\times V\),
\[
\|P(\vy,\vx)-P(\vy',\vx')\|
=
\|f_{i_*}(\vx)-f_{i_*}(\vx')\|
\le
L\|\vx-\vx'\|
\le
L\|(\vy,\vx)-(\vy',\vx')\|.
\]
Hence \(P\) is locally Lipschitz near \((\vy_0,\vx_0)\). Since
\((\vy_0,\vx_0)\) was arbitrary, \(P\) is locally Lipschitz on $(\cY\times\overline{\cK})\setminus \cB$.
\end{proof}

\subsection{Construction of dynamics}

\textbf{Distance-to-bad-set safety factor.}
For \((\vy,\vx)\in \cY\times\overline \cK\), set
\[
\rho(\vy,\vx):=\operatorname{dist}((\vy,\vx),\mathcal B),
\qquad
\omega(\vy,\vx)
:=
\frac{\rho(\vy,\vx)}{1+\rho(\vy,\vx)}.
\]
Because \(\cB\) is closed,
\[
\omega(\vy,\vx)=0
\quad\Longleftrightarrow\quad
(\vy,\vx)\in\cB.
\]
and we consider $(\vy,\vx)$ to be a safe point if this factor $\omega$ is large enough. 
Moreover, for every fixed \(\vy\in\mathcal Y\), the map
\(\vx\mapsto \omega(\vy,\vx)\) is \(1\)-Lipschitz on
\(\mathcal K\). Indeed,
\[
\left|\rho(\vy,\vx)-\rho(\vy,\vz)\right|\le
\left\|(\vy,\vx)-(\vy,\vz)\right\|_{\mathcal Y\times\overline{\mathcal K}} =
\|\vx-\vz\|,
\]
and \(r\mapsto r/(1+r)\) is \(1\)-Lipschitz on \([0,\infty)\). Hence
\[
\left|\omega(\vy,\vx)-\omega(\vy,\vz)\right|
\le \|\vx-\vz\|
\qquad
\text{for all } \vx,\vz\in\mathcal K .
\]

\textbf{Safe tubes and safe-set profiles.}
Define the safe tube (i.e., a clear gap to the bad set):
\[
\cS(R,s):= \Big\{(\vy,\vx)\in\cY\times\overline{\cK}:~~
\|\vy\|\le R, ~\omega(\vy,\vx)\ge s.
\Big\}
\]
And then define safe-set profiles:
\[
H_1(R,s) := \sup_{\substack{(\vy,\vx) \in \cS(R,s) \\ (\vy,\vz) \in \cS(R,s) \\ \vx\neq \vz}}\frac{\|P_\vy(\vx)-P_\vy(\vz)\|}{\|\vx-\vz\|}.
\]
and 
\[
H_2(R,s) := \sup_{(\vy,\vx) \in \cS(R,s)}\|P_\vy(\vx)\|.
\]

\begin{lemma}
It holds that 
\[
H_1(R,s)<\infty, \qquad H_2(R,s)<\infty
\]
for every \(R>0\) and \(s>0\).
\end{lemma}

\begin{proof}
If \(s\ge 1\), then the safe sets are empty since \(\omega<1\), and the
claim is trivial. Hence assume \(0<s<1\).
Since \(\overline B_R^{\cY}\times\overline{\cK}\) is compact and \(\omega\)
is continuous, \(\cS(R,s)\) is compact. Moreover,
\[
\cS(R,s)\subset (\cY\times\overline{\cK})\setminus \cB .
\]
Thanks to Lemma \ref{lemma:lip-p-operator-joint}, $P(\vy,\vx)$ is joint local Lipschitz on $(\cY\times\overline{\cK})\setminus \cB$ and the restriction of \(P\) to the compact set \(\cS(R,s)\) is globally Lipschitz. Hence there exists
\(L_{R,s}<\infty\) such that
\[
\|P_\vy(\vx)-P_\vy(\vz)\|
=
\|P(\vy,\vx)-P(\vy,\vz)\|
\le
L_{R,s}\|(\vy,\vx)-(\vy,\vz)\|
=
L_{R,s}\|\vx-\vz\|
\]
for all
\((\vy,\vx),(\vy,\vz)\in \cS(R,s)\). Therefore $H_1(R,s)\le L_{R,s}<\infty$.

Since \(P\) is continuous on the compact set \(\cS(R,s)\), it is
bounded there. Thus there exists \(M_{R,s}<\infty\):
\[
\|P_\vy(\vx)\|\le M_{R,s}
\qquad
\text{for all }(\vy,\vx)\in \cS(R,s).
\]
Therefore, $H_2(R,s)\le M_{R,s}<\infty$. This completes the proof.
\end{proof}

\textbf{Damping factor and recurrent operator.}
Define
\[
\lambda_R(s)
:=
\frac{1}{H_1(R,s)+H_2(R,s)+1},
\qquad s>0.
\]
Define
\[
\Theta_R(r):=\int_0^r \lambda_R(s)\,ds,
\qquad r\ge0.
\]
Define
\[
\eta(\vy,\vx)
:=
\frac{\Theta_{1+\|\vy\|}(\omega(\vy,\vx))}
{1+\Theta_{1+\|\vy\|}(\omega(\vy,\vx))}.
\]
Finally define 
\begin{equation}
\label{eq:define-g-operator}
g(\vy,\vx):=(1-\eta(\vy,\vx))\vy+\eta(\vy,\vx)P_\vy(\vx).
\end{equation}

\begin{lemma}[Monotonicity]
\label{lemma:monotonicity}
For fixed \(R>0\), \(H_1(R,s)\) and \(H_2(R,s)\) are nonincreasing in \(s\),
and hence \(\lambda_R(s)\) is nondecreasing in \(s\). Moreover, for fixed
\(s>0\), \(H_1(R,s)\) and \(H_2(R,s)\) are nondecreasing in \(R\), and hence
\(\lambda_R(s)\) is nonincreasing in \(R\).
\end{lemma}

\begin{proof}
Since the safe tube $\cS(R,s)$ shrinks as \(s\) increases, \(H_1(R,s)\) and \(H_2(R,s)\) are nonincreasing in \(s\), \(\lambda_R(s)\) is nondecreasing in \(s\). Similar argument for \(R\).
\end{proof}

\begin{lemma}[Positivity and freezing]
\label{lemma:freezing}
For \((\vy,\vx)\in \cY\times\overline \cK\),
\[
\eta(\vy,\vx)=0
\iff
(\vy,\vx)\in\cB
\iff
\vx\in \cD\cup\cU^{\overline \cK}
~~\text{ or }~~ \vy\in\widetilde\Sigma(\vx).
\]
In particular, for \(\vx\in \cK\), since
\(\cD\cap \cK=\varnothing\) and
\(\cU^{\overline \cK}\cap \cK=\cU\), we have
\[
\vx\in\cU \implies g(\vy,\vx)=\vy.
\]
\end{lemma}

\begin{proof}
By definition,
\[
\eta(\vy,\vx)
=
\frac{\Theta_{1+\|\vy\|}(\omega(\vy,\vx))}
     {1+\Theta_{1+\|\vy\|}(\omega(\vy,\vx))}.
\]
Since \(\lambda_R(s)>0\) for every \(s>0\), we have
\(\Theta_R(r)=0\) if and only if \(r=0\). Hence
\[
\eta(\vy,\vx)=0
\iff
\omega(\vy,\vx)=0.
\]
Moreover,
\[
\omega(\vy,\vx)=0
\iff
\rho(\vy,\vx)=0
\iff
(\vy,\vx)\in\cB,
\]
where the last equivalence uses the closedness of \(\cB\). By the definition of
\(\cB\), this is equivalent to
\[
\vx\in \cD\cup\cU^{\overline \cK}
\quad\text{or}\quad
\vy\in\widetilde\Sigma(\vx).
\]
Finally, if \(\vx\in\cU\), then \(\vx\in\cU^{\overline \cK}\), so
\((\vy,\vx)\in\cB\) for every \(\vy\). Thus \(\eta(\vy,\vx)=0\), and therefore
\[
g(\vy,\vx)
=
(1-\eta(\vy,\vx))\vy+\eta(\vy,\vx)P_\vy(\vx)
=
\vy.
\]
This completes the proof.
\end{proof}

\subsection{Convergence given fixed x}

\begin{lemma}[Graph-closure switching agrees with reduced switching at stable points]
\label{lemma:closure-equal-reduce}
Fix \(\vx\in \cK\setminus \cU\). Then
\[
\widetilde\Sigma(\vx)=\Sigma^\circ(\vx).
\]
\end{lemma}

\begin{proof} 
The inclusion $\Sigma^\circ(\vx)\subseteq \widetilde\Sigma(\vx)$ is immediate from the definition of graph closure. We prove the reverse inclusion. 

Let \(\vq\in\widetilde\Sigma(\vx)\). By definition, there exist
\(\vx_\ell\to \vx\) and \(\vq_\ell\to \vq\) such that
\[
\vq_\ell\in \Sigma^\circ(\vx_\ell).
\]
For each \(\ell\), choose labels \(i_\ell\neq j_\ell\) witnessing the switching, namely
\[
f_{i_\ell}(\vx_\ell)\neq f_{j_\ell}(\vx_\ell),
\]
and
\[
\|\vq_\ell-f_{i_\ell}(\vx_\ell)\| = \|\vq_\ell-f_{j_\ell}(\vx_\ell)\| = \min_{1\le k\le n}\|\vq_\ell-f_k(\vx_\ell)\|.
\]
Since there are only finitely many label pairs, after passing to a subsequence we may assume
\[
i_\ell=i,\qquad j_\ell=j
\]
for fixed \(i\neq j\).

Because \(\vx\in \cK\setminus\mathcal U\), the collapse pattern is stable near \(\vx\). Hence, after
shrinking to a neighborhood of \(\vx\), whether two labels are collapsed is constant. Since
\[
f_i(\vx_\ell)\neq f_j(\vx_\ell)
\]
for infinitely many \(\ell\), the labels \(i\) and \(j\) cannot belong to the same stable collapse
cluster. Therefore
\[
f_i(\vx)\neq f_j(\vx).
\]
Now let \(\ell\to\infty\). By continuity of the branch maps,
\[
\|\vq-f_i(\vx)\| = \|\vq-f_j(\vx)\| = \min_{1\le k\le n}\|\vq-f_k(\vx)\|.
\]
Together with \(f_i(\vx)\neq f_j(\vx)\), this implies $\vq\in\Sigma^\circ(\vx)$. Therefore, $\widetilde\Sigma(\vx)\subseteq \Sigma^\circ(\vx)$, and hence $\widetilde\Sigma(\vx)=\Sigma^\circ(\vx)$.
\end{proof}

\begin{lemma}[Voronoi-segment invariance]
\label{lemma:segment-invariance}
Fix \(\vx\in \cK\setminus\cU\). Let $\vy_0\notin\Sigma^\circ(\vx)$, and define $\vu:=P_{\vy_0}(\vx)$. Then for every \(\vz\in[\vy_0,\vu]\) (Here, $[\vy_0,\vu]$ represents the line segment connecting the endpoints $\vy_0$ and $\vu$.),
\[
P_\vz(\vx)=\vu.
\]
In particular,
\[
[\vy_0,\vu]\cap \widetilde\Sigma(\vx)=\varnothing.
\]
\end{lemma}

\begin{proof}
Since \(\vx\in \cK\setminus\mathcal U\), Lemma \ref{lemma:closure-equal-reduce} gives $\widetilde\Sigma(\vx)=\Sigma^\circ(\vx)$ and hence 
\[
\vy_0\notin\widetilde\Sigma(\vx).
\]
Hence the selected value \(\vu=P_{\vy_0}(\vx)\) is the unique nearest branch value to \(\vy_0\).

Let \(\vv\in \cF(\vx)\) be any distinct value with \(\vv\neq \vu\). Then
\[
\|\vy_0-\vu\|<\|\vy_0-\vv\|.
\]
For \(t\in[0,1]\), set
\[
\vz_t:=(1-t)\vy_0+t\vu.
\]
Then
\[
\|\vz_t-\vu\|=(1-t)\|\vy_0-\vu\|.
\]
By the triangle inequality,
\[
\|\vz_t-\vv\| \ge \|\vy_0-\vv\|-\|\vy_0-\vz_t\| = \|\vy_0-\vv\|-t\|\vy_0-\vu\|.
\]
Since \(\|\vy_0-\vv\|>\|\vy_0-\vu\|\), we obtain
\[
\|\vz_t-\vv\| > (1-t)\|\vy_0-\vu\| = \|\vz_t-\vu\|.
\]
Therefore \(\vu\) remains the unique nearest branch value for every
\(\vz_t\in[\vy_0,\vu]\). Hence
\[
P_{\vz_t}(\vx)=\vu
\]
for every \(t\in[0,1]\), and then $\vz_t\notin \Sigma^\circ(\vx)$ for every \(t\in[0,1]\).

Using again $\widetilde\Sigma(\vx)=\Sigma^\circ(\vx)$, we conclude that
\[
\vz_t\notin \widetilde\Sigma(\vx)
\]
for every \(t\in[0,1]\). Therefore $[\vy_0,\vu]\cap \widetilde\Sigma(\vx)=\varnothing$, which concludes the proof.
\end{proof}

\begin{lemma}[Uniform positive bad-set distance along the Voronoi segment]
\label{lemma:uniform-dist-in-cell}
Fix \(\vx\in \cK\setminus\mathcal U\). Let
\(\vy_0\notin\Sigma^\circ(\vx)\) and define $\vu:=P_{\vy_0}(\vx)$.
Then
\[
\inf_{\vz\in[\vy_0,\vu]}\rho(\vz,\vx)>0.
\]
\end{lemma}

\begin{proof}
Define
\[
\Gamma:=\{(\vz,\vx):\vz\in[\vy_0,\vu]\}
\subset \cY\times\overline{\cK}.
\]
The set \(\Gamma\) is compact because \([\vy_0,\vu]\) is compact.

We claim that
\[
\Gamma\cap\mathcal B=\varnothing .
\]
Indeed, since \(\vx\in\cK\setminus\mathcal U\) and
\(\mathcal D\subset\overline{\cK}\setminus\cK\), we have $\vx\notin \mathcal U\cup\mathcal D$.
Moreover, by Lemma~\ref{lemma:segment-invariance}, $[\vy_0,\vu]\cap\widetilde\Sigma(\vx)=\varnothing$.
Therefore, for every \(\vz\in[\vy_0,\vu]\), the point
\((\vz,\vx)\) belongs to none of the components defining \(\mathcal B\).
Hence \(\Gamma\cap\mathcal B=\varnothing\).

Since \(\mathcal B\) is closed and \(\Gamma\) is compact, the continuous
function
\[
(\vz,\vx)\mapsto \operatorname{dist}((\vz,\vx),\mathcal B)
\]
attains a strictly positive minimum on \(\Gamma\). Therefore
\[
\inf_{\vz\in[\vy_0,\vu]}\rho(\vz,\vx)
=
\operatorname{dist}(\Gamma,\mathcal B)
>0.
\]
This proves the claim.
\end{proof}

\begin{lemma}[Positive damping]
\label{lemma:positive-damping}
Fix \(\vx\in \cK\setminus\mathcal U\). Let $\vy_0\notin\Sigma^\circ(\vx)$ and define $\vu:=P_{\vy_0}(\vx)$. 
Then
\[
c:=\inf_{\vz\in[\vy_0,\vu]}\eta(\vz,\vx)>0.
\]
\end{lemma}

\begin{proof}
By Lemma \ref{lemma:uniform-dist-in-cell}, there exists \(\rho_0>0\) such that
\[
\rho(\vz,\vx)\ge \rho_0
\]
for every \(\vz\in[\vy_0,\vu]\). Hence
\[
\omega(\vz,\vx) = \frac{\rho(\vz,\vx)}{1+\rho(\vz,\vx)} \ge \frac{\rho_0}{1+\rho_0} =:\omega_0>0.
\]
Let
\[
R_0:=1+\max_{\vz\in[\vy_0,\vu]}\|\vz\|.
\]
Then
\[
1+\|\vz\|\le R_0
\]
for every \(\vz\in[\vy_0,\vu]\). Since \(H_1(R,s)\) and \(H_2(R,s)\) are
nondecreasing in \(R\), the function \(\lambda_R(s)\) is nonincreasing in \(R\).
Therefore
\[
\Theta_{1+\|\vz\|}(\omega_0)
\ge
\Theta_{R_0}(\omega_0).
\]
Thus, for every \(\vz\in[\vy_0,\vu]\),
\[
\eta(\vz,\vx)
\ge
\frac{\Theta_{R_0}(\omega_0)}
{1+\Theta_{R_0}(\omega_0)}
=:c.
\]
Since \(\omega_0>0\) and \(\lambda_{R_0}(s)>0\) for every \(s>0\),
\[
\Theta_{R_0}(\omega_0)
=
\int_0^{\omega_0}\lambda_{R_0}(s)\,ds
>0.
\]
Hence \(c>0\).
\end{proof}

\begin{theorem}[Convergence with fixed $\vx$]
\label{thm:convergence-fix-x}
Fix \(\vx\in \cK\setminus\mathcal U\). Let $\vy_0\notin\Sigma^\circ(\vx)$ and define $\vu:=P_{\vy_0}(\vx)$. 
Consider the iteration
\[
\vy_{t+1}=g(\vy_t,\vx).
\]
Then
\[
\vy_t\to \vu.
\]
More precisely, there exists \(c>0\) such that
\[
\|\vy_t-\vu\|\le (1-c)^t\|\vy_0-\vu\|.
\]
where $c$ may depend on $\vx$ and $\vy_0$.
\end{theorem}

\begin{proof}By Lemma \ref{lemma:positive-damping},
\[
c:=\inf_{\vz\in[\vy_0,\vu]}\eta(\vz,\vx)>0.
\]
By Lemma \ref{lemma:segment-invariance},
\[
P_\vz(\vx)=\vu
\]
for every \(\vz\in[\vy_0,\vu]\).

We prove by induction that
\[
\vy_t\in[\vy_0,\vu]
\]
for every \(t\). This is true for \(t=0\). If \(\vy_t\in[\vy_0,\vu]\), then
\[
\vy_{t+1} = (1-\eta(\vy_t,\vx))\vy_t+\eta(\vy_t,\vx)\vu.
\]
Since
\[
0\le\eta(\vy_t,\vx)\le1,
\]
we have
\[
\vy_{t+1}\in[\vy_t,\vu]\subset[\vy_0,\vu].
\]
Thus \(\vy_t\in[\vy_0,\vu]\) for all \(t\).

Therefore
\[
\vy_{t+1}-\vu = (1-\eta(\vy_t,\vx))(\vy_t-\vu),
\]
and so
\[
\|\vy_{t+1}-\vu\|\le(1-c)\|\vy_t-\vu\|.
\]
Iterating gives
\[
\|\vy_t-\vu\| \le (1-c)^t\|\vy_0-\vu\|.
\]
Therefore \(\vy_t\to \vu\).
\end{proof}

\subsection{Input regularity}

\begin{lemma} 
\label{lemma:lip-damping}
Fix $\vy \in \cY$. The map $\vx\mapsto \eta(\vy,\vx)P_\vy(\vx)$ is globally Lipschitz on \(\overline{\cK}\).
\end{lemma}

\begin{proof}
Fix \(\vy\). 
Set \(R:=1+\|\vy\|\). Then \(\|\vy\|\le R\), and
\[
\eta(\vy,\vx)=
\frac{\Theta_R(\omega(\vy,\vx))}
{1+\Theta_R(\omega(\vy,\vx))}.
\]
For simplicity, we write
\[
\omega(\vx):=\omega(\vy,\vx),
\qquad
P(\vx):=P_\vy(\vx),
\qquad
\eta(\vx):=\eta(\vy,\vx).
\]
Since \(\vx\mapsto\omega(\vx)\) is \(1\)-Lipschitz, we have
\[
|\omega(\vx)-\omega(\vz)|\le \|\vx-\vz\|.
\]
\textbf{Case 1:}
If \((\vy,\vx),(\vy,\vz)\in \cB\), according to Lemma \ref{lemma:freezing}, we have $\eta(\vx)=\eta(\vz)=0$ and hence
\[
\eta(\vx)P(\vx)=\eta(\vz)P(\vz)=0.
\]
\textbf{Case 2:}
If \((\vy,\vx)\in \cB\) and \((\vy,\vz)\notin \cB\), set
\[
r:=\omega(\vz)>0.
\]
Then \((\vy,\vz)\in\mathcal \cS({R,r})\), so
\[
\|P(\vz)\|\le H_2(R,r).
\]
Also,
\[
\eta(\vz)\le \Theta_R(r)\le r\lambda_R(r).
\]
Hence
\[
\|\eta(\vz)P(\vz)\| \le r\lambda_R(r)H_2(R,r) \le r.
\]
Since
\[
r=\omega(\vz)-\omega(\vx)\le\|\vz-\vx\|,
\]
we get
\[
\|\eta(\vx)P(\vx)-\eta(\vz)P(\vz)\| \le \|\vx-\vz\|.
\]
\textbf{Case 3:}
Now suppose \((\vy,\vx),(\vy,\vz)\notin \cB\). Without loss of generality, assume
\[
r_1:=\omega(\vx)\le r_2:=\omega(\vz).
\]
Then \((\vy,\vz),(\vy,\vx)\in\mathcal \cS({R,r_1})\), so
\[
\|P(\vx)-P(\vz)\| \le H_1(R,r_1)\|\vx-\vz\|.
\]
Also,
\[
\|P(\vz)\|\le H_2(R,r_2).
\]
Therefore
\[
\big\|\eta(\vx)P(\vx)-\eta(\vz)P(\vz)\big\| \le \eta(\vx)\|P(\vx)-P(\vz)\| + |\eta(\vx)-\eta(\vz)| \big\|P(\vz)\big\|.
\]
For the first term,
\[
\eta(\vx)\le\Theta_R(r_1)\le r_1\lambda_R(r_1),
\]
so
\[
\eta(\vx)\|P(\vx)-P(\vz)\| \le r_1\lambda_R(r_1)H_1(R,r_1)\|\vx-\vz\| \le \|\vx-\vz\|.
\]
For the second term, since \(\lambda_R\) is nondecreasing,
\[
|\eta(\vx)-\eta(\vz)| \le |\Theta_R(r_1)-\Theta_R(r_2)| \le (r_2-r_1)\lambda_R(r_2).
\]
Thus
\[
|\eta(\vx)-\eta(\vz)| \big\|P(\vz)\big\| \le (r_2-r_1)\lambda_R(r_2)H_2(R,r_2) \le r_2-r_1.
\]
Since \(\omega\) is \(1\)-Lipschitz,
\[
r_2-r_1 \le \|\vx-\vz\|.
\]
Combining the two estimates gives
\[
\|\eta(\vx)P(\vx)-\eta(\vz)P(\vz)\| \le 2\|\vx-\vz\|.
\]
Therefore \(\vx\mapsto\eta(\vy,\vx)P_\vy(\vx)\) is globally Lipschitz on \(\overline\cK\).
\end{proof}

\begin{theorem}[Global Lipschitzness of the damped operator]
\label{thm:lip-g}
For every fixed \(\vy\in \cY\), the map
\[
\vx\mapsto g(\vy,\vx)
\]
is globally Lipschitz on \(\overline\cK\).
\end{theorem}

\begin{proof}
We write
\[
g(\vy,\vx) = \vy-\eta(\vy,\vx)\vy+\eta(\vy,\vx)P_\vy(\vx).
\]
By Lemma \ref{lemma:lip-damping}, $\vx\mapsto \eta(\vy,\vx)P_\vy(\vx)$ is globally Lipschitz.

It remains to control \(\vx\mapsto\eta(\vy,\vx)\vy\). Since \(\vy\) is fixed, it is enough to control \(\vx\mapsto\eta(\vy,\vx)\). Because \(r\mapsto r/(1+r)\) is \(1\)-Lipschitz and
\[
\Theta_R'(r)=\lambda_R(r)\le1
\]
for a.e. \(r\), we have
\[
|\eta(\vy,\vx)-\eta(\vy,\vz)| \le |\omega(\vy,\vx)-\omega(\vy,\vz)| \le
\|\vx-\vz\|.
\]
Therefore
\[
\|\eta(\vy,\vx)\vy-\eta(\vy,\vz)\vy\| \le \|\vy\|\|\vx-\vz\|.
\]
Combining the two estimates, \(\vx\mapsto g(\vy,\vx)\) is globally Lipschitz on \(\overline\cK\).
\end{proof}

\subsection{Final proof of Theorem \ref{thm:expressivity}}

Finally, we restate Theorem \ref{thm:expressivity} in a more complete version and prove it.

\begin{theorem}
Let \(\cK\subset \mathbb{R}^d\) be bounded and let \(\cY=\mathbb{R}^m\) with the Euclidean norm. Assume $f_1,\dots,f_n:\cK\to \cY$ be locally Lipschitz. These branches can collapse. Then there exists $g:\cY\times\cK\to\cY$ satisfying:
\begin{enumerate}
\item For every \(\vy\in \cY\), the map $\vx\mapsto g(\vy,\vx)$ is globally Lipschitz on \(\cK\).

\item For every \(\vx\in\cU\), $g(\vy,\vx)=\vy$ for every \(\vy\in \cY\). Hence the iteration freezes at unstable collapse points.

\item For every \(\vx\in \cK\setminus\mathcal U\) and every $\vy_0\notin\Sigma^\circ(\vx)$, the iteration $\vy_{t+1}=g(\vy_t,\vx)$ converges to one of the branch in $\cF(\vx)$:
\[
\vy_t \to P_{\vy_0}(\vx)\in \cF(\vx)
\]
Moreover, $\Sigma^\circ(\vx)$ has Lebesgue measure zero in \(\cY\).
Therefore, for every \(\vx\in \cK\setminus\cU\), the convergence statement holds for Lebesgue-almost every initialization \(\vy_0\in \cY\).

\item For every \(\vx\in \cK\setminus\cU\), it holds that 
\[
\left\{
\lim_{t\to\infty}\vy_t(\vx;\vy_0): \vy_0\in \cY\setminus\Sigma^\circ(\vx)
\right\}
= \cF(\vx).
\]
In other words, every branch in $\cF(\vx)$ can be achieved
\end{enumerate}
\end{theorem}

\begin{proof}
The first claim follows from Theorem \ref{thm:lip-g}.

The second claim follows from Lemma \ref{lemma:freezing}.

The third claim follows from Theorem \ref{thm:convergence-fix-x}. In addition, For a stable \(\vx\), the graph-closure switching set agrees locally with the ordinary reduced Voronoi switching set (Lemma \ref{lemma:closure-equal-reduce}). Since \(\cF(\vx)\) is finite,
\(\Sigma^\circ(\vx)\) is contained in a finite union of pairwise bisector
hyperplanes. Hence \(\Sigma^\circ(\vx)\) has Lebesgue measure zero in \(\cY\).

For the fourth claim, fix \(\vx\in \cK\setminus\mathcal U\). By the third claim, every initialization \(\vy_0\notin\Sigma^\circ(\vx)\) converges to some element of \(\cF(\vx)\). Hence
\[
\left\{
\lim_{t\to\infty}\vy_t(\vx;\vy_0):
\vy_0\in \cY\setminus\Sigma^\circ(\vx)
\right\}
\subset \cF(\vx).
\]
Conversely, take any \(\vu\in \cF(\vx)\), and initialize at
\[
\vy_0:=\vu.
\]
Then \(\vu\) is the unique nearest point value to itself among the distinct branch values in \(\cF(\vx)\). Therefore
\[
\vy_0\notin\Sigma^\circ(\vx),
\qquad
P_{\vy_0}(\vx)=\vu.
\]
By Theorem \ref{thm:convergence-fix-x}, the iteration converges to \(\vu\). Hence every
\(\vu\in \cF(\vx)\) is obtained as a limit, proving equality.
\end{proof}

\section{Proof of Theorem \ref{thm:main-lip-selector}}
\label{app:theorem-2}

We first describe Theorem \ref{thm:main-lip-selector} in a more precise version and then provide its complete proof.

\begin{theorem}
\label{thm:app-lip-selector}
Let \(\cK\subset \mathbb{R}^d\) be bounded and measurable, let
\(\cY=\mathbb{R}^m\), \(m\ge 1\), and let $f_1,\ldots,f_n:\cK\to \cY$ be locally Lipschitz branch maps. Let $\cK^\ast := \cK\setminus\cU$.
Let \(g:\cY\times \cK\to \cY\) be the recurrent operator constructed in Theorem \ref{thm:expressivity}. 
For \(\vy_0\in \cY\), define the induced solution selector
\[
u(\vy_0,\vx):=\lim_{t\to\infty}g^t(\vy_0,\vx),
\]
whenever the limit exists.

Then there exists a Lebesgue-null set \(\cE\subset \cY\) such that for every
\(\vy_0\in \cY\setminus \cE\), the fixed-initialization switching set
\[
\cN(\vy_0) := \{\vx\in \cK^\star: \vy_0\in \Sigma^\circ(\vx)\}
\]
has zero \(d\)-dimensional Lebesgue measure:
\[
\mathcal{L}^d(\cN(\vy_0))=0.
\]
Moreover, for every \(\vy_0\in \cY\setminus \cE\), the limit $u(\vy_0,\vx)$ exists for all $\vx\in \cK^\star\setminus \cN(\vy_0)$, and the map $\vx\mapsto u(\vy_0,\vx)$ is locally Lipschitz on \(\cK^\star\setminus \cN(\vy_0)\).
\end{theorem}

Note: In Appendix~\ref{app:theorem-2}, we write $\mathcal L^d$ for
$d$-dimensional Lebesgue measure to avoid ambiguity in the Fubini argument.
Elsewhere, we use $|\cdot|$ for Lebesgue measure in the relevant ambient space.

\begin{proof}
We first show that for almost every fixed initialization \(\vy_0\), the
switching set \(\cN(\vy_0)\) has zero measure.
For fixed \(\vx\in \cK^\star\), the switching set \(\Sigma^\circ(\vx)\) is
contained in a finite union of affine hyperplanes in \(\cY\). Indeed, for
each \(i\neq j\) with \(f_i(\vx)\neq f_j(\vx)\), expanding $\|\vy-f_i(\vx)\|^2=\|\vy-f_j(\vx)\|^2$ gives
\[
2\langle \vy,f_j(\vx)-f_i(\vx)\rangle = \|f_j(\vx)\|^2-\|f_i(\vx)\|^2.
\]
Since \(f_i(\vx)\neq f_j(\vx)\), the vector \(f_j(\vx)-f_i(\vx)\) is nonzero,
and hence this equation defines a genuine affine hyperplane in
\(\cY=\mathbb{R}^m\). Therefore it has \(m\)-dimensional Lebesgue measure
zero. Since there are only finitely many pairs \((i,j)\), we obtain
\[
\mathcal{L}^m(\Sigma^\circ(\vx))=0
\qquad
\text{for every }\vx\in \cK^\star.
\]

Now define the switching incidence set in the joint space
\[
\Gamma := \{(\vx,\vy)\in \cK^\star\times \cY:\ \vy\in\Sigma^\circ(\vx)\}.
\]
Equivalently,
\[
\Gamma = \bigcup_{1\le i<j\le n} \Gamma_{ij},
\]
where \(\Gamma_{ij}\) is the set of all \((\vx,\vy)\in \cK^\star\times \cY\)
such that
\[
f_i(\vx)\neq f_j(\vx), \qquad \|\vy-f_i(\vx)\| = \|\vy-f_j(\vx)\| = \min_{1\le k\le n}\|\vy-f_k(\vx)\|,
\]
Each \(\Gamma_{ij}\) is measurable because it is described by finitely
many equalities and inequalities involving continuous functions of
\((\vx,\vy)\). Hence \(\Gamma\) is measurable.

For \(R>0\), define the closed radius-\(R\) ball in the initialization
space \(\cY\) by
\[
\sB_R^\cY:=\{\vy\in \cY:\|\vy\|\le R\}.
\]
We use \(\sB_R^\cY\) only to localize the argument to a finite-measure set:
since \(\cK\) is bounded and measurable, \(\cK^\star\subset \cK\) has finite
Lebesgue measure, and \(\sB_R^\cY\) also has finite Lebesgue measure. Thus
the indicator function
\[
\mathbf{1}_{\Gamma\cap (\cK^\star\times \sB_R^\cY)}
\]
is integrable on \(\cK^\star\times \sB_R^\cY\), so Fubini's theorem applies.

For fixed \(\vx\in \cK^\star\), define the vertical fiber
\[
\Gamma_\vx^R := \{\vy\in \sB_R^\cY:(\vx,\vy)\in\Gamma\}.
\]
By definition,
\[
\Gamma_\vx^R = \Sigma^\circ(\vx)\cap \sB_R^\cY.
\]
Therefore
\[
\mathcal{L}^m(\Gamma_\vx^R) = \mathcal{L}^m(\Sigma^\circ(\vx)\cap \sB_R^\cY) = 0 \qquad \text{for every }\vx\in \cK^\star.
\]
Applying Fubini's theorem \cite[Theorem 1.22]{evans2015measure} in the order \(y\)-then-\(x\), we get
\[
\mathcal{L}^{d+m}\bigl(\Gamma\cap(\cK^\star\times \sB_R^\cY)\bigr)
= \int_{\cK^\star} \left( \int_{\sB_R^\cY} \mathbf{1}_{\Gamma}(\vx,\vy)\,\mathrm d \vy \right)\mathrm d\vx = \int_{\cK^\star} \mathcal{L}^m(\Gamma_\vx^R)\,\mathrm d\vx  = 0
\]
Next apply Fubini's theorem in the other order. For fixed
\(\vy\in \sB_R^\cY\), define the horizontal fiber
\[
\Gamma_R^\vy := \{\vx\in \cK^\star:(\vx,\vy)\in\Gamma\}.
\]
By definition,
\[
\Gamma_R^\vy = \{\vx\in \cK^\star:\vy\in\Sigma^\circ(\vx)\} = \cN(\vy).
\]
Therefore,
\[
\begin{aligned}
0
= \mathcal{L}^{d+m}\bigl(\Gamma\cap(\cK^\star\times \sB_R^\cY)\bigr)
&= \int_{\sB_R^\cY} \left( \int_{\cK^\star} \mathbf{1}_{\Gamma}(\vx,\vy)\,\mathrm d\vx \right)\mathrm d\vy  \\
&= \int_{\sB_R^\cY} \mathcal{L}^d(\Gamma_R^\vy)\,\mathrm d \vy  = \int_{\sB_R^\cY} \mathcal{L}^d(\cN(\vy))\,\mathrm d \vy.
\end{aligned}
\]
The function $\vy\mapsto \mathcal{L}^d(\cN(\vy))$ is nonnegative. Hence the last identity implies that
\[
\mathcal{L}^d(\cN(\vy))=0
\]
for Lebesgue-a.e. \(\vy\in \sB_R^\cY\).

Let \(\cE_R\subset \sB_R^\cY\) be the exceptional set
\[
\cE_R := \{\vy\in \sB_R^\cY:\mathcal{L}^d(\cN(\vy))>0\}.
\]
Then
\[
\mathcal{L}^m(\cE_R)=0.
\]
Now define
\[
\cE:=\bigcup_{\ell=1}^{\infty}\cE_\ell,
\]
where \(\cE_\ell\) is the exceptional set corresponding to the ball
\(\sB_\ell^\cY\). Since \(\cE\) is a countable union of Lebesgue-null sets,
\[
\mathcal{L}^m(\cE)=0.
\]
Take any \(\vy_0\in \cY\setminus \cE\). Choose an integer \(\ell\) large enough
so that \(\vy_0\in \sB_\ell^\cY\). Since \(\vy_0\notin \cE_\ell\), we have
\[
\mathcal{L}^d(\cN(\vy_0))=0.
\]
This proves the desired almost-everywhere statement for fixed
initializations.

It remains to prove the local Lipschitz property. Fix
\(\vy_0\in \cY\setminus \cE\) and \(\vx_0\in \cK^\star\setminus \cN(\vy_0)\). Since \(\vx_0\notin \cN(\vy_0)\), we have $\vy_0\notin\Sigma^\circ(\vx_0)$. Thus, among the distinct branch values in
\[
\cF(\vx_0)=\{f_1(\vx_0),\ldots,f_n(\vx_0)\},
\]
there is a unique branch value closest to \(\vy_0\).

Because \(\vx_0\in \cK^\star\), the branch collapse pattern is stable near
\(\vx_0\). Therefore there exist a relative neighborhood
\(\sV_0\subset \cK\) of \(\vx_0\) and a partition
\[
\{1,\ldots,n\}=C_1\cup\cdots\cup C_q
\]
such that, for every \(\vx\in \sV_0\),
\[
f_i(\vx)=f_j(\vx)
\quad\Longleftrightarrow\quad
i,j\in C_\alpha
\text{ for some }\alpha .
\]
Choose one representative \(i_\alpha\in C_\alpha\) for each cluster
\(C_\alpha\). Since \(\vy_0\notin\Sigma^\circ(\vx_0)\), there is a unique
cluster \(C_{\alpha_\star}\) whose branch value is closest to \(\vy_0\) at
\(\vx_0\). That is,
\[
\|\vy_0-f_{i_{\alpha_\star}}(\vx_0)\| < \|\vy_0-f_{i_\beta}(\vx_0)\|
\qquad
\text{for every }\beta\neq \alpha_\star.
\]
If \(q=1\), then there is only one distinct branch value near \(\vx_0\),
and the conclusion is immediate. Otherwise, define the positive margin
\[
\delta := \min_{\beta\neq \alpha_\star}
\left[ \|\vy_0-f_{i_\beta}(\vx_0)\| - \|\vy_0-f_{i_{\alpha_\star}}(\vx_0)\|
\right] >0.
\]
By continuity of the finitely many branch maps, after shrinking \(\sV_0\)
if necessary, we may ensure that for every \(\vx\in \sV_0\) and every
\(\beta\neq\alpha_\star\),
\[
\|\vy_0-f_{i_{\alpha_\star}}(\vx)\| < \|\vy_0-f_{i_\beta}(\vx)\|.
\]
Thus the same cluster \(C_{\alpha_\star}\) remains the unique nearest
branch cluster throughout \(\sV_0\). Consequently,
\[
P_{\vy_0}(\vx)=f_{i_{\alpha_\star}}(\vx) \qquad \text{for every }\vx\in \sV_0.
\]
In particular, it holds that
\[
\sV_0\cap \cN(\vy_0)=\emptyset,
\]
and equivalently
\[ \vy_0 \notin \Sigma^\circ(\vx),\qquad \text{for every }\vx\in \sV_0. \]
By Theorem \ref{thm:convergence-fix-x}, for every \(\vx\in \sV_0\),
\[
u(\vy_0,\vx) = \lim_{t\to\infty}g^t(\vy_0,\vx) = P_{\vy_0}(\vx) = f_{i_{\alpha_\star}}(\vx).
\]
Since \(f_{i_{\alpha_\star}}\) is locally Lipschitz, the map $\vx\mapsto u(\vy_0,\vx)$ is locally Lipschitz in a neighborhood of \(\vx_0\).
Because \(\vx_0\in \cK^\star\setminus \cN(\vy_0)\) was arbitrary, hence \(\vx\mapsto u(\vy_0,\vx)\) is locally Lipschitz on
\(\cK^\star\setminus \cN(\vy_0)\).
\end{proof}

\textbf{An example to understand Theorem \ref{thm:app-lip-selector}.}
Consider the one-dimensional set-valued map
\[
F(x)=\{x,-x\},
\qquad x\in \mathcal K\subset\mathbb R,
\]
with branches \(f_1(x)=x\) and \(f_2(x)=-x\). The two branches collapse at
\(x=0\), and this collapse is not stable: in every neighborhood of \(0\),
the branches are equal at \(0\) but distinct away from \(0\). Hence the
stable domain is
\[
\mathcal K^\star=\mathcal K\setminus\{0\}.
\]
Now fix an initialization \(y_0\). If \(y_0=0\), then for every
\(x\in\mathcal K^\star\),
\[
|y_0-x|=|y_0+x|,
\]
so \(y_0\) lies exactly on the Voronoi switching boundary for every stable
input. Therefore
\[
\cN(0)=\mathcal K^\star.
\]
Thus \(y_0=0\) is a bad initialization: it is switching for a positive-measure,
indeed full-measure, set of stable inputs. However, this bad initialization is
exceptional. If \(y_0\neq 0\), then
\[
|y_0-x|=|y_0+x|
\]
implies
\[
(y_0-x)^2=(y_0+x)^2,
\]
and hence
\[
4y_0x=0.
\]
Since \(y_0\neq 0\), this forces \(x=0\), which is not contained in
\(\mathcal K^\star\). Therefore
\[
\cN(y_0)=\varnothing
\qquad
\text{for every }y_0\neq 0.
\]
This example shows why the theorem must exclude a null set of initializations:
some special initialization, here \(y_0=0\), can be bad for many inputs, but the
set of such initializations has Lebesgue measure zero in \(Y=\mathbb R\). For
almost every fixed initialization \(y_0\), the induced selector is stable and
locally Lipschitz on the whole stable domain \(\mathcal K^\star\).

\section{Proof of Theorem \ref{thm:arbitrary-irregular-selector}}

We now restate Theorem \ref{thm:arbitrary-irregular-selector} to facilitate readability and proceed to prove it. 

\begin{theorem}
Let $\cK^\star\subset \mathbb{R}^d$ be bounded and measurable, and let $\cY=\mathbb{R}^m$. Let $f_1,\ldots,f_n:\cK^\star\to\cY$ be locally Lipschitz branch maps.
Define the genuinely multi-branch region
\[
    \cM
    :=
    \{\vx\in\cK^\star:|\cF(\vx)|\ge 2\}
    =
    \{\vx\in\cK^\star:\exists i\neq j
    \text{ such that } f_i(\vx)\neq f_j(\vx)\}.
\]
Then, for every
\[
    a\in[0,|\cM|],
\]
there exists a measurable solution selector
\[
    u_a:\cK^\star\to\cY,
    \qquad
    u_a(\vx)\in\cF(\vx),
\]
such that
\[
    \left|
        \operatorname{Disc}_{\cK^\star}(u_a)
    \right|
    =
    a.
\]
Here $\operatorname{Disc}_{\cK^\star}(u_a)$ denotes the set of
discontinuity points of $u_a$ relative to the domain $\cK^\star$.
\end{theorem}

\begin{proof}
For each pair $1\le i<j\le n$, define
\[
    \cO_{ij}
    :=
    \{\vx\in\cK^\star:f_i(\vx)\neq f_j(\vx)\}.
\]
Since $f_i$ and $f_j$ are continuous relative to $\cK^\star$,
the set $\cO_{ij}$ is relatively open in $\cK^\star$. Moreover,
\[
    \cM
    =
    \bigcup_{1\le i<j\le n}\cO_{ij}.
\]
Thus $\cM$ is relatively open in $\cK^\star$ and is Lebesgue measurable.

\textbf{Step 1: Scope of discontinuity.}
We first note that no selector can be discontinuous outside $\cM$.
Indeed, let $\vx_0\in\cK^\star\setminus\cM$. Then all branch values agree:
\[
    f_1(\vx_0)=\cdots=f_n(\vx_0)=:\vy_0.
\]
Let $u$ be any selector satisfying $u(\vx)\in\cF(\vx)$. For any sequence
$\vx_k\to\vx_0$ in $\cK^\star$, there exists an index
$i_k\in\{1,\ldots,n\}$ such that
\[
    u(\vx_k)=f_{i_k}(\vx_k).
\]
Hence
\[
    \|u(\vx_k)-\vy_0\|
    =
    \|f_{i_k}(\vx_k)-f_{i_k}(\vx_0)\|
    \le
    \max_{1\le i\le n}
    \|f_i(\vx_k)-f_i(\vx_0)\|
    \to 0.
\]
Therefore $u$ is continuous at $\vx_0$, and consequently
\[
    \operatorname{Disc}_{\cK^\star}(u)\subset \cM
\]
for every solution selector $u$.

\textbf{Step 2: Construction of disjoint open cover.}
Now fix $a\in[0,|\cM|]$. If $a=0$, choose
\[
    u_a(\vx)=f_1(\vx).
\]
Since $f_1$ is locally Lipschitz, it is continuous, and hence
\[
    \left|
        \operatorname{Disc}_{\cK^\star}(u_a)
    \right|
    =
    0.
\]

It remains to consider $a>0$. Since $\cM$ is relatively open in
$\cK^\star$, for every $\vx\in\cM$ there exists a pair $i<j$ and an
open ball $B(\vx,r)$ such that
\[
    B(\vx,r)\cap\cK^\star
    \subset
    \cO_{ij}.
\]
The family of all such ambient balls is a Vitali cover of $\cM$.
By the Vitali covering theorem \cite[Theorem 1.24]{evans2015measure}, there exists a countable pairwise
disjoint collection of open balls
\[
    B_1,B_2,\ldots
\]
such that, for each $\ell$, there is a pair $(i_\ell,j_\ell)$ with
\[
    G_\ell:=B_\ell\cap\cK^\star
    \subset
    \cO_{i_\ell j_\ell},
\]
and
\begin{equation}
\label{eq:vitali-cover}
\left| \cM\setminus \bigcup_{\ell=1}^{\infty}G_\ell \right| = 0.
\end{equation}
Since the balls $B_\ell$ are pairwise disjoint, the sets
$G_\ell$ are pairwise disjoint. Therefore
\[
    \sum_{\ell=1}^{\infty}|G_\ell|
    =
    \left|
        \bigcup_{\ell=1}^{\infty}G_\ell
    \right|
    =
    |\cM|.
\]

We now construct a relatively open set $G_a\subset\cM$ with $|G_a|=a$.

If $0<a<|\cM|$, define partial sums
\[
    S_N:=\sum_{\ell=1}^{N}|G_\ell|,
    \qquad
    S_0:=0.
\]
Since $S_N\to|\cM|>a$, there exists a first index $N$ such that $S_N>a$. Hence
\[
    S_{N-1}\le a<S_N.
\]
Set $r:=a-S_{N-1}$. Then $0\le r<|G_N|$. If $r=0$, just define
\[
    G_a:=\bigcup_{\ell=1}^{N-1}G_\ell.
\]
If $r>0$, let
\[
    H_t:=\{\vx\in\mathbb{R}^d:x_1<t\}.
\]
The function $\phi(t):=|G_N\cap H_t|$ is continuous, satisfies $\phi(t)\to 0$ as $t\to-\infty$, and satisfies $\phi(t)\to |G_N|$ as $t\to+\infty$. Hence there exists $t_0$ such that $|G_N\cap H_{t_0}|=r$. Now define
\[
    G_a
    :=
    \left(
        \bigcup_{\ell=1}^{N-1}G_\ell
    \right)
    \cup
    (G_N\cap H_{t_0}).
\]
Then $G_a\subset\cM$, $G_a$ is relatively open in $\cK^\star$, and
\[
    |G_a|=a.
\]
Moreover, its relative boundary satisfies
\[
    |\partial_{\cK^\star}G_a|=0,
\]
because $\partial_{\cK^\star}G_a$ is contained in a finite union of sets
of the form $\cK^\star\cap\partial B_\ell$ and possibly one hyperplane
slice $\cK^\star\cap\{x_1=t_0\}$, all of which have $d$-dimensional
Lebesgue measure zero.

If $a=|\cM|$, define
\[
    G_a:=\bigcup_{\ell=1}^{\infty}G_\ell.
\]
Then $G_a\subset\cM$, $G_a$ is relatively open in $\cK^\star$, and
\[
    |\cM\setminus G_a|=0,
    \qquad
    |G_a|=|\cM|=a.
\]
In this case, $\partial_{\cK^\star}G_a\cap\cM \subset \cM\setminus G_a$ because $G_a$ is relatively open and $G_a\subset\cM$. Then by \eqref{eq:vitali-cover}, we have
\[
    |\partial_{\cK^\star}G_a\cap\cM|=0.
\]

\textbf{Step 3: Construction of solution selector.}
We next define the selector $u_a$. The set $G_a$ is a finite or countable
union of disjoint selected pieces. Each selected piece $E$ is of one of
the forms
\[
    E=G_\ell
    \qquad\text{or}\qquad
    E=G_N\cap H_{t_0},
\]
and is contained in some $\cO_{i_Ej_E}$ for a pair $i_E<j_E$.
For each such selected piece $E$, choose a countable dense subset
$Q_E\subset E$ with respect to the relative topology of $E$.
Define
\[
    u_a(\vx)
    =
    \begin{cases}
        f_{i_E}(\vx), & \vx\in Q_E,\\[1mm]
        f_{j_E}(\vx), & \vx\in E\setminus Q_E,
    \end{cases}
    \qquad \vx\in E.
\]
Outside $G_a$, define
\[
    u_a(\vx)=f_1(\vx).
\]
Because the selected pieces are measurable and at most countably many,
and because each $Q_E$ is countable, the map $u_a$ is measurable.
Moreover,
\[
    u_a(\vx)\in\{f_1(\vx),\ldots,f_n(\vx)\}=\cF(\vx)
\]
for every $\vx\in\cK^\star$.

\textbf{Step 4: Discontinuity set.}
We now estimate the discontinuity set of $u_a$.

First, let $E$ be one selected piece. Since $E$ is measurable, the Lebesgue density theorem \cite[Theorem 1.35]{evans2015measure} implies that almost every point of $E$ is a density point of $E$. Let $\vx_0$ be such a density point. Because $Q_E$ is dense in $E$, there exists a sequence
\[
    \vx_k\in Q_E,
    \qquad
    \vx_k\to\vx_0.
\]
Also, since $Q_E$ is countable and $\vx_0$ is a density point of $E$,
every sufficiently small relative neighborhood of $\vx_0$ contains
points of $E\setminus Q_E$. Hence there exists a sequence
\[
    \vz_k\in E\setminus Q_E,
    \qquad
    \vz_k\to\vx_0.
\]
Along these two sequences,
\[
    u_a(\vx_k)=f_{i_E}(\vx_k)\to f_{i_E}(\vx_0),
\]
while
\[
    u_a(\vz_k)=f_{j_E}(\vz_k)\to f_{j_E}(\vx_0).
\]
Since $E\subset\cO_{i_Ej_E}$, we have
\[
    f_{i_E}(\vx_0)\neq f_{j_E}(\vx_0).
\]
Thus $u_a$ is discontinuous at $\vx_0$. Therefore $u_a$ is discontinuous
at almost every point of every selected piece $E$, and hence $|G_a\setminus \operatorname{Disc}_{\cK^\star}(u_a)|=0$. Consequently,
\[
    \left|
        \operatorname{Disc}_{\cK^\star}(u_a)
    \right|
    \ge
    |G_a|
    =
    a.
\]

Conversely, we already showed that no selector can be discontinuous
outside $\cM$. Moreover, if
\[
    \vx_0\in \cM\setminus\left(G_a\cup\partial_{\cK^\star}G_a\right),
\]
then $\vx_0$ has a relative neighborhood in $\cK^\star$ disjoint from
$G_a$. On that neighborhood, $u_a=f_1$, so $u_a$ is continuous at
$\vx_0$. Therefore
\[
    \operatorname{Disc}_{\cK^\star}(u_a)
    \subset
    G_a\cup(\partial_{\cK^\star}G_a\cap\cM).
\]

If $0<a<|\cM|$, then $|\partial_{\cK^\star}G_a|=0$, so
\[
    \left|
        \operatorname{Disc}_{\cK^\star}(u_a)
    \right|
    \le
    |G_a|
    =
    a.
\]
Together with the lower bound, this gives
\[
    \left|
        \operatorname{Disc}_{\cK^\star}(u_a)
    \right|
    =
    a.
\]

If $a=|\cM|$, then $|\partial_{\cK^\star}G_a\cap\cM|=0$, so again
\[
    \left|
        \operatorname{Disc}_{\cK^\star}(u_a)
    \right|
    =
    |G_a|
    =
    |\cM|
    =
    a.
\]
This completes the proof.
\end{proof}

\section{Experiment Details for Toy Examples}
\label{app:toy-experiments}

This appendix records the implementation details for the toy-example experiments in Section~\ref{sec:theory}. In particular, it specifies the parameterization of the weight-tied dynamics, the construction of the training/validation/test data, and the staged training procedure used in the numerical verifications for Theorems~\ref{thm:expressivity}--\ref{thm:arbitrary-irregular-selector}.

\textbf{Model.}
In all toy-example experiments, we parameterize the weight-tied dynamics $g_\theta$ by a two-layer MLP with embedding width $256$. We use the same parameterization across the two toy examples. The model is trained in the unrolled form used throughout the paper, with shared parameters across all iterations.

\textbf{Data.}
The first toy example \textbf{(Algebraic)} is constructed from the two branches of the map $x \mapsto \{\pm 1/|x|\}$. Its training inputs are sampled on a geometric grid
\[
\mathcal{X}_{\mathrm{train}}^{(1)}=\{\pm (1+\delta)^s\},
\]
with $\delta=0.1$, truncated to the range determined by $x_{\max}=10$ and $y_{\max}=10$; in particular, the grid excludes $x=0$. The validation inputs are generated independently on a different geometric grid with step size $0.2$, so the validation points do not coincide with the training points. For visualization and testing, we use denser grids that are again generated independently from the training grid.

The second toy example \textbf{(Double-well)} is constructed from the two stable branches of the double-well equation
\[
4y^3-4y+x=0.
\]
Its training inputs are sampled uniformly on $[-2.5,2.5]$ with $500$ points, while the validation inputs are sampled independently on the same interval with $128$ points. For visualization and testing, we again use denser independent grids. Therefore, in both toy examples, the train and test data are independent.

\textbf{Training.}
We largely follow Algorithm~\ref{alg:toy-supervised}, with one practical modification: we use staged training with increasing unroll lengths. For the first toy example, the three stages use unroll lengths
\[
[10,20,30],
\]
and for the second toy example, they use
\[
[3,6,9].
\]
All runs use learning rate $10^{-4}$, $5000$ epochs per stage, and $3$ stages in total. The initialization distribution is
\[
y_0 \sim \mathrm{Unif}[-5,5],
\]
and for each input we sample $10$ initializations during training. In the supervised toy-example experiments, when multiple target branches are available for the same input, we assign each sampled initialization $y_0$ to the branch that is closest to $y_0$ in Euclidean distance.

\textbf{Optimization and model selection.}
We use the same optimizer settings across the two toy examples, together with the stage-wise continuation described above. Within each stage, model selection is based on the validation loss, and the best checkpoint in that stage is used to initialize the next stage. This produces the final supervised checkpoints used in the toy-example visualizations.

\textbf{Branch selector construction.}
For the branch-selection experiments in Section~\ref{sec:multi-branch-benefit}, we derive single-valued supervised datasets from the original two-branch toy examples by imposing an interval-wise alternating selector. Specifically, we partition the input domain into $N_{\rm int}$ intervals and assign one branch to each interval, with adjacent intervals always selecting different branches. In the first toy example, the partition is uniform in $\log |x|$, so that branch switching is distributed evenly across scales near the singular region. In the second toy example, the partition is uniform in $x$. We consider interval counts $N_{\rm int} \in \{2,4,8,16,32,64\}$. This construction keeps the underlying solution set unchanged while producing increasingly irregular single-valued targets as $N_{\rm int}$ grows, thereby isolating the effect of arbitrary human branch selection from the effect of the original multi-solution geometry.

\textbf{Hardware.} All neural training and evaluation runs in this paper were performed on a single workstation with four NVIDIA Quadro RTX~6000 GPUs (24\,GB each).

\section{Ising-GNN Experiments: Details and Addtional Results}
\label{app:ising-scaling}

\subsection{Experiment Details}
This subsection provides more experiment details as a supplement to the main text.

\textbf{Dataset construction.}
Each graph is a triangular lattice generated by python package \texttt{networkx}, with all couplings antiferromagnetic: $J_{ij}\sim -\mathrm{Unif}[0.5,1.5]$.
This choice makes every edge prefer opposite spins, while the triangular geometry prevents all constraints from being satisfied simultaneously. The result is a geometrically frustrated family of Ising instances with many competing low-energy patterns.
The stored dataset contains $2500$ training graphs and $500$ held-out test graphs. The graph sizes vary substantially: the number of nodes ranges from $6$ to $55$ with mean $27.37$, and the number of undirected edges ranges from $9$ to $135$ with mean about $62.08$. 

For supervised comparisons, we use Gurobi on \eqref{eq:physical-energy} to compute spin label $\vs\in\{-1,+1\}^{|\cV|}$. This label is only one branch among many physically admissible low-energy branches.

\textbf{Graph models.} We largely adopt the implicit GNN architecture from \cite{liu2026expressive}, unrolling the dynamics for $T=10$ iterations. However, when training the supervised baseline against exact Gurobi labels, we observed that the standard update was highly unstable and struggled to minimize the loss. To achieve reliable convergence, we modify the supervised model to use a damped recurrent update:
$$\vs_{t+1} = (1-\alpha) \vs_t + \alpha\,g_{\rm GNN}(\vs_t, \cG), \qquad \alpha \in (0,1].$$
This damping mitigates the training instability associated with forcing the weight-tied dynamics to fit a single target branch. We treat $\alpha$ as a hyperparameter and tune it over the grid $\{0.01, 0.05, 0.1, 0.5, 1\}$. For the vanilla GNN counterpart, we use the codebase from \cite{chen2023on}.

\textbf{Other details.}
All models use hidden width $h=64$. For the original small-scale dataset, models are trained with Adam, learning rate $10^{-3}$, batch size $32$, and a maximum of $10{,}000$ epochs.

\subsection{Additional Results: Distribution of the Solution Number}

While Table~\ref{tab:ising-accuracy} in the main text summarizes the average number of distinct solutions discovered by the weight-tied GNN dynamics, this appendix provides the complete distributional breakdown. Figure~\ref{fig:ising-hist} demonstrates that the multi-branch behavior is consistently expressed across the dataset: rather than being skewed by a few extreme outliers, the model successfully recovers a high multiplicity of distinct low-energy states for the vast majority of frustrated graph instances.

\begin{figure}[t]
    \centering
    \includegraphics[width=0.65\linewidth]{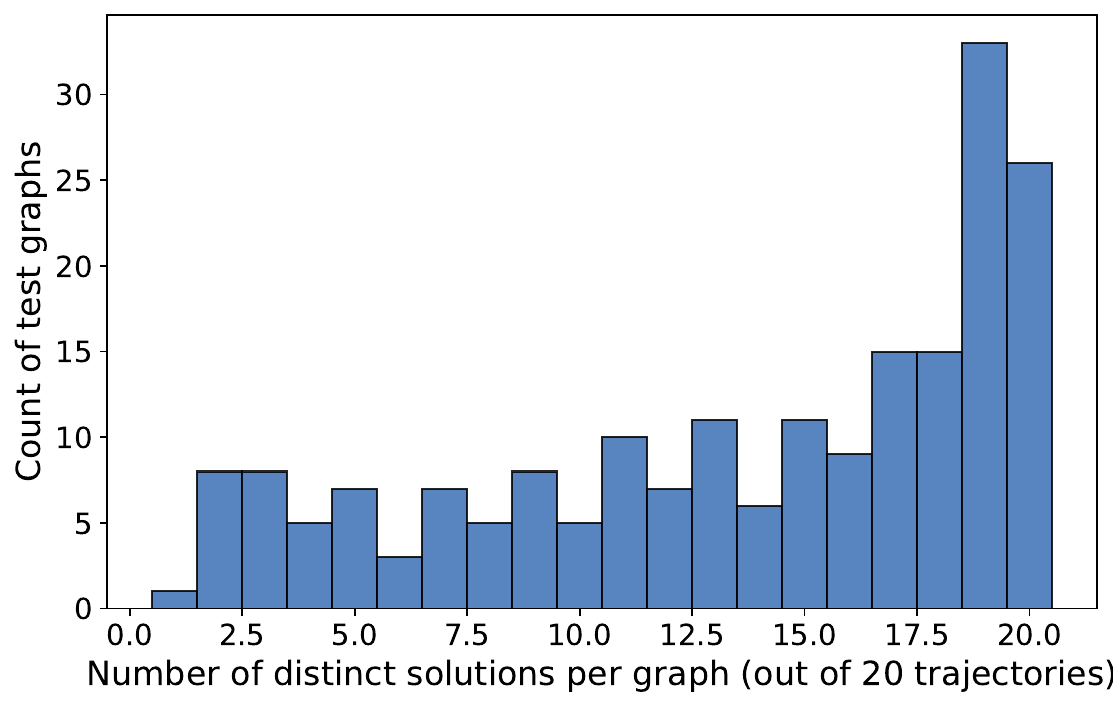}
    \caption{Histogram of the number of distinct rounded solutions per test graph for the weight-tied GNN dynamic. Evaluation uses $200$ test graphs and $k=20$ independent initializations per graph.}
    \label{fig:ising-hist}
\end{figure}

\subsection{Additional Results: Scaling up the Graph Size}

This subsection reports the larger-scale Ising experiments used to assess how the models behave as the graph size increases. All datasets in this subsection use the same triangular antiferromagnetic construction as the base dataset in the main text. Model, training, and evaluation protocols are also the same. The only difference is the graph size.

\textbf{Dataset levels.}
We generated eight dataset levels, denoted $L1$ through $L8$, by increasing the target graph-size range in the same data-generation script used for the small dataset. Each level contains $2500$ training graphs and $500$ test graphs. Table~\ref{tab:ising-level-sizes} reports the realized graph-size statistics.

\begin{table*}
\centering
\caption{Realized graph-size statistics for the eight triangular antiferromagnetic dataset levels used in the scaling experiments. Edge counts are undirected counts.}
\label{tab:ising-level-sizes}
\small
\setlength{\tabcolsep}{5pt}
\begin{tabular}{c c c c c c c}
\toprule
Level & Split & \# Graphs & Node Range & Mean Nodes & Edge Range & Mean Edges \\
\midrule
L1 & Train & 2500 & $66$--$120$ & $86.64$ & $165$--$315$ & $222.05$ \\
L1 & Test  & 500  & $66$--$120$ & $86.05$ & $165$--$315$ & $220.40$ \\
\midrule
L2 & Train & 2500 & $136$--$253$ & $186.28$ & $360$--$693$ & $502.65$ \\
L2 & Test  & 500  & $136$--$253$ & $186.40$ & $360$--$693$ & $503.01$ \\
\midrule
L3 & Train & 2500 & $276$--$528$ & $388.27$ & $759$--$1488$ & $1083.06$ \\
L3 & Test  & 500  & $276$--$496$ & $387.77$ & $759$--$1395$ & $1081.62$ \\
\midrule
L4 & Train & 2500 & $561$--$990$ & $751.35$ & $1584$--$2838$ & $2139.65$ \\
L4 & Test  & 500  & $561$--$946$ & $756.40$ & $1584$--$2709$ & $2154.40$ \\
\midrule
L5 & Train & 2500 & $990$--$1953$ & $1476.19$ & $2838$--$5673$ & $4267.92$ \\
L5 & Test  & 500  & $990$--$1953$ & $1474.49$ & $2838$--$5673$ & $4262.84$ \\
\midrule
L6 & Train & 2500 & $2016$--$3916$ & $2979.14$ & $5859$--$11484$ & $8708.41$ \\
L6 & Test  & 500  & $2016$--$3916$ & $2989.31$ & $5859$--$11484$ & $8738.53$ \\
\midrule
L7 & Train & 2500 & $4005$--$7875$ & $5961.49$ & $11748$--$23250$ & $17559.91$ \\
L7 & Test  & 500  & $4005$--$7875$ & $5876.83$ & $11748$--$23250$ & $17308.27$ \\
\midrule
L8 & Train & 2500 & $8001$--$15931$ & $11975.45$ & $23625$--$47259$ & $35465.78$ \\
L8 & Test  & 500  & $8001$--$15931$ & $12218.36$ & $23625$--$47259$ & $36189.83$ \\
\bottomrule
\end{tabular}
\end{table*}

\textbf{Interpretation.}
Table~\ref{tab:ising-emb64-all-levels} shows that energy-based training scales much more reliably than Gurobi-label fitting. 
Across all levels $L1$--$L8$, the energy-based models remain trainable on increasingly large frustrated graphs, and the recurrent GNN consistently gives the lowest rounded-spin energies; even at $L8$, where test graphs have about $1.2\times 10^4$ nodes and $3.6\times 10^4$ edges, its test energy remains around $-1.16$.

In contrast, Gurobi-label training is available only for $L1$--$L5$ due to labeling cost, and its performance degrades rapidly as graph size grows. 
This is consistent with the main-text interpretation: fitting one selected branch becomes increasingly brittle in large multi-solution landscapes, whereas energy-based training only asks the model to find any low-energy valid branch.

\begin{table*}
\caption{Ising results across dataset levels $L1$--$L8$. Supervised results are available only for $L1$--$L5$, since Gurobi labeling takes $>50$ hrs for L6 labeling. }
\label{tab:ising-emb64-all-levels}
\centering
\small
\setlength{\tabcolsep}{4pt}
\begin{tabular}{l l l rr}
\toprule
Dataset Level & Objective & Model & Train energy & Test energy \\
\midrule
L1 & Gurobi labeled & Vanilla GNN & $0.4413 \pm 0.5470$ & $0.4474 \pm 0.5096$  \\
L1 & Gurobi labeled & Recurrent GNN & $-0.1735 \pm 0.2357$ & $-0.1634 \pm 0.2294$  \\
L1 & Energy based & Vanilla GNN & $-0.8761 \pm 0.0751$ & $-0.8765 \pm 0.0701$  \\
L1 & Energy based & Recurrent GNN & $-1.0612 \pm 0.0407$ & $-1.0638 \pm 0.0422$  \\
\midrule
L2 & Gurobi labeled & Vanilla GNN & $0.7460 \pm 0.3787$ & $0.7644 \pm 0.4004$  \\
L2 & Gurobi labeled & Recurrent GNN & $0.3858 \pm 0.2314$ & $0.3875 \pm 0.2295$  \\
L2 & Energy based &Vanilla GNN & $-0.9003 \pm 0.0532$ & $-0.9017 \pm 0.0497$  \\
L2 & Energy based &Recurrent GNN & $-1.1185 \pm 0.0276$ & $-1.1165 \pm 0.0291$ \\
\midrule
L3 & Gurobi labeled & Vanilla GNN & $2.2180 \pm 0.2777$ & $2.2055 \pm 0.2773$  \\
L3 & Gurobi labeled & Recurrent GNN & $1.0121 \pm 0.2629$ & $0.9885 \pm 0.2653$  \\
L3 & Energy based &Vanilla GNN & $-0.9181 \pm 0.0389$ & $-0.9167 \pm 0.0390$  \\
L3 & Energy based &Recurrent GNN & $-1.1458 \pm 0.0189$ & $-1.1457 \pm 0.0179$  \\
\midrule
L4 & Gurobi labeled & Vanilla GNN & $2.5518 \pm 0.1508$ & $2.5516 \pm 0.1463$  \\
L4 & Gurobi labeled & Recurrent GNN & $1.7608 \pm 0.1986$ & $1.7606 \pm 0.1809$  \\
L4 & Energy based &Vanilla GNN & $-0.9352 \pm 0.0259$ & $-0.9356 \pm 0.0264$  \\
L4 & Energy based &Recurrent GNN & $-1.1504 \pm 0.0139$ & $-1.1508 \pm 0.0139$  \\
\midrule
L5 & Gurobi labeled & Vanilla GNN & $2.5416 \pm 0.0850$ & $2.5427 \pm 0.0817$  \\
L5 & Gurobi labeled & Recurrent GNN & $2.4534 \pm 0.0923$ & $2.4482 \pm 0.0877$  \\
L5 & Energy based &Vanilla GNN & $-0.9448 \pm 0.0199$ & $-0.9435 \pm 0.0202$  \\
L5 & Energy based &Recurrent GNN & $-1.1464 \pm 0.0101$ & $-1.1461 \pm 0.0096$  \\
\midrule
L6 & Gurobi labeled & Vanilla GNN & -- & --  \\
L6 & Gurobi labeled & Recurrent GNN & -- & --  \\
L6 & Energy based &Vanilla GNN & $-0.9540 \pm 0.0139$ & $-0.9531 \pm 0.0137$ \\
L6 & Energy based &Recurrent GNN & $-1.1554 \pm 0.0071$ & $-1.1554 \pm 0.0071$ \\
\midrule
L7 & Gurobi labeled & Vanilla GNN & -- & -- \\
L7 & Gurobi labeled & Recurrent GNN & -- & -- \\
L7 & Energy based &Vanilla GNN & $-0.9583 \pm 0.0097$ & $-0.9577 \pm 0.0100$  \\
L7 & Energy based &Recurrent GNN & $-1.1695 \pm 0.0080$ & $-1.1698 \pm 0.0082$  \\
\midrule
L8 & Gurobi labeled & Vanilla GNN & -- & --  \\
L8 & Gurobi labeled & Recurrent GNN & -- & --  \\
L8 & Energy based &Vanilla GNN & $-0.9616 \pm 0.0067$ & $-0.9617 \pm 0.0068$ \\
L8 & Energy based &Recurrent GNN & $-1.1605 \pm 0.0075$ & $-1.1605 \pm 0.0074$  \\
\bottomrule
\end{tabular}
\end{table*}

\section{Allen-Cahn Experiments: More Details and Additional Results}
\label{app:ac-additional}

\subsection{Experiment Details}

This subsection provides more experiment details as a supplement to the main text.

\textbf{Dataset construction.}
All experiments in this section use the fixed interface parameter \(\varepsilon=0.01\).
Each forcing field is generated on an \(N\times N\) periodic grid over \(\Omega=(0,2\pi)^2\). Unless otherwise stated, \(N=64\).
For each instance, we sample a periodic Gaussian random field (GRF). In particular, we draw independent Gaussian Fourier noise, impose the conjugate-symmetry condition needed to obtain a real-valued spatial field, set the zero mode to zero, and multiply each Fourier mode \(k\) by the
filter
\[
  \exp(-\|k\|_2^2/(2\ell^2))(1+\|k\|_2^2)^{-\alpha/2},
\]
where the smoothness and spectral-scale parameters are sampled independently as
\(\alpha\sim\mathrm{Unif}[1.4,2.8]\) and \(\ell\sim\mathrm{Unif}[8.0,14.0]\).
After applying the inverse FFT, we subtract the empirical spatial mean and divide by the empirical
spatial standard deviation, obtaining a normalized field \(\tilde f\). The final forcing is
\(f=A\tilde f+m\), where
\(A\sim\mathrm{Unif}[0.08,0.20]\) controls the amplitude and
\(m\sim\mathrm{Unif}[-0.03,0.03]\) controls the mean shift.
The dataset contains \(1000\) training instances, \(1000\) validation instances, and \(1000\)
held-out test instances, generated independently using the same procedure.

\textbf{Model structure.}
We largely follow the implicit FNO architecture of~\cite{marwah2023deep}. In that framework, a neural operator is used as a shared update rule for a steady-state solver rather than as a one-shot feed-forward predictor. At each iteration, the model takes the current state estimate together with the fixed PDE input, applies an FNO-based update block, and reuses the same parameters across all unrolled iterations. This gives a weight-tied dynamics whose equilibrium is interpreted as the predicted steady solution.
In our Allen--Cahn experiments, the input is the normalized forcing field \(f\), and the recurrent state is the phase-field estimate \(u_t\). We use the residual update
\[
  u_{t+1}
  =
  g_\theta(u_t,f)
  =
  u_t + \eta \Phi_\theta(u_t,f),
  \qquad \eta = 0.2,
\]
where \(\Phi_\theta\) is an FNO-based state-dependent update field shared across all iterations, following \cite{marwah2023deep}. 
  In particular, we choose such a block which has depth one and consists of one spectral convolution, one pointwise \(1\times1\) convolution, additive forcing injection, and a GELU nonlinearity. 
  For the main \(64\times64\) experiments, all models use width \(32\), \(32\times32\) retained Fourier modes, one recurrent update block, one internal layer inside that block, and one forcing input channel. 
  
  The initial recurrent state \(u_0\) is sampled from a GRF distribution with scale \(0.5\), smoothness parameter \(\alpha=3.0\), and spectral scale \(\ell=8.0\). During training, each forcing instance is replicated with \(8\) independent GRF initializations, while validation and test-time evaluation use \(16\) independent GRF initializations per instance. This trains the same weight-tied operator to evolve from multiple candidate starting states for the same physical input.

\textbf{Supervised label making.}
For supervised training, we construct labels by solving the steady Allen--Cahn equation with a classical periodic Implicit-Explicit (IMEX) solver \cite{chen1998applications} on the same spatial grid as the dataset. This is a commonly used method to solve Allen--Cahn. For each forcing instance, the classical solver is initialized from a GRF draw generated by the same initialization procedure used in model training, namely with scale \(0.5\), smoothness parameter \(\alpha=3.0\), and spectral scale \(\ell=8.0\). The solver tolerance is fixed at \(10^{-5}\). The resulting steady solutions are used as the supervised targets in all supervised Allen--Cahn experiments reported in the main text and appendix.

\textbf{Solution counting.} To count the number of distinct solution branches discovered by the model, we cluster the $M$ final states based on their pairwise root mean squared error (RMSE). Predictions are grouped into the same branch if their distance falls below a resolution-specific threshold. To ensure a fair baseline, this threshold is empirically calibrated such that the unregularized model (where the diversity penalty is zero) naturally collapses to a single dominant branch for the vast majority of test instances.

\textbf{Training.}
We largely follow Algorithm~\ref{alg:unsupervised}, but use a fixed six-stage curriculum in the unroll length. Allmodels are trained for \(120\) epochs with the Adam optimizer, weight decay \(10^{-4}\), batch size \(16\), and initial learning rate \(10^{-3}\). The six curriculum stages start at epochs \(0,20,40,60,80,100\), with training unroll lengths \(2,4,8,12,16,20\), respectively. The learning rate is kept constant within each stage and decreased across stages according to the schedule \([10^{-3},8\times10^{-4},6\times10^{-4},4\times10^{-4},3\times10^{-4},2\times10^{-4}]\). During training, each forcing instance is paired with \(8\) independent GRF initial states. During validation and testing, we use \(16\) independent GRF initial states per forcing instance. At each epoch, the model is evaluated on the validation set using the current curriculum stage, with the validation rollout length set to twice the training rollout length at that stage. We select the checkpoint with the smallest validation residual and report its final test performance. For supervised runs, we keep the same optimizer, epoch budget, learning-rate schedule, and unroll curriculum, but replace the unsupervised objective by a pure supervised regression loss to the chosen labeled branch.

\textbf{Diversity-weight curriculum.}
The diversity coefficient is stage-dependent rather than fixed throughout training. We use the same
six rollout stages as the unroll-length curriculum, with rollout lengths \(2,4,8,12,16,20\) respectively.
For each diversity-regularized run, we start with a relatively large diversity weight to encourage
separation among trajectories from different initial states, and then decay the weight so that later
stages focus more on physical refinement. In the last stage the diversity weight is set to zero, allowing
the model to improve the Allen--Cahn energy without further forcing separation. We report each
schedule in the main text by its initial value \(\lambda_{\rm init}\). The schedules are
\[
\begin{aligned}
\lambda_{\rm init}=0 &: [0,0,0,0,0,0],\\
\lambda_{\rm init}=0.5 &: [0.5,0.2,0.06,0.02,0.006,0],\\
\lambda_{\rm init}=1.2 &: [1.2,0.5,0.16,0.05,0.016,0],\\
\lambda_{\rm init}=2.0 &: [2.0,0.9,0.28,0.09,0.028,0].
\end{aligned}
\]
This curriculum separates the two roles of the objective: early stages encourage multiple basins of
attraction, while later stages refine the discovered states toward lower-energy Allen--Cahn solutions.

\subsection{Additional Experiments: ML-initialized Solver}
\label{app:ac-convergence-compare}

We first compare three test-time solvers on the canonical $64\times 64$ problem:
\begin{enumerate}
  \item \textbf{pure ML}: the learned weight-tied dynamic trained with $\lambda=0$ (energy only),
  \item \textbf{numerical solver}: the classical periodic IMEX solver,
  \item \textbf{hybrid solver}: a hybrid scheme that first applies the
  learned solver and then hands the iterate to the IMEX solver.
\end{enumerate}
The hybrid solver uses the learned model for at most the first $50$ steps. Let
\[
  r_t^{\mathrm{fp}}
  =
  \left(
    \frac{1}{n^2}
    \sum_{p,q}
    \bigl(u_{t+1}(p,q)-u_t(p,q)\bigr)^2
  \right)^{1/2}
\]
denote the fixed-point residual of the learned recurrence on an $n\times n$ grid. The handoff criterion is $r_t^{\mathrm{fp}} \le 10^{-2}$.
If this never occurs, the switch is forced at step $50$. After the switch, the IMEX solver continues from the learned state using the same total budget of $200$ recorded iterations. 

\textbf{Result.}
Figure~\ref{fig:ac-n64-ai-num-hybrid} reports the resulting energy and residual
trajectories. The learned model decreases the variational objective rapidly, but
its residual error plateaus. The numerical solver reduces the residual more
steadily, but may remain at a higher energy for a substantial part of the
trajectory. The hybrid method combines the two effects: it exploits the
learned solver to reach a low-energy basin quickly and then uses the numerical
solver to refine the state toward a smaller residual.

\begin{figure}[t]
  \centering
  \includegraphics[width=\textwidth]{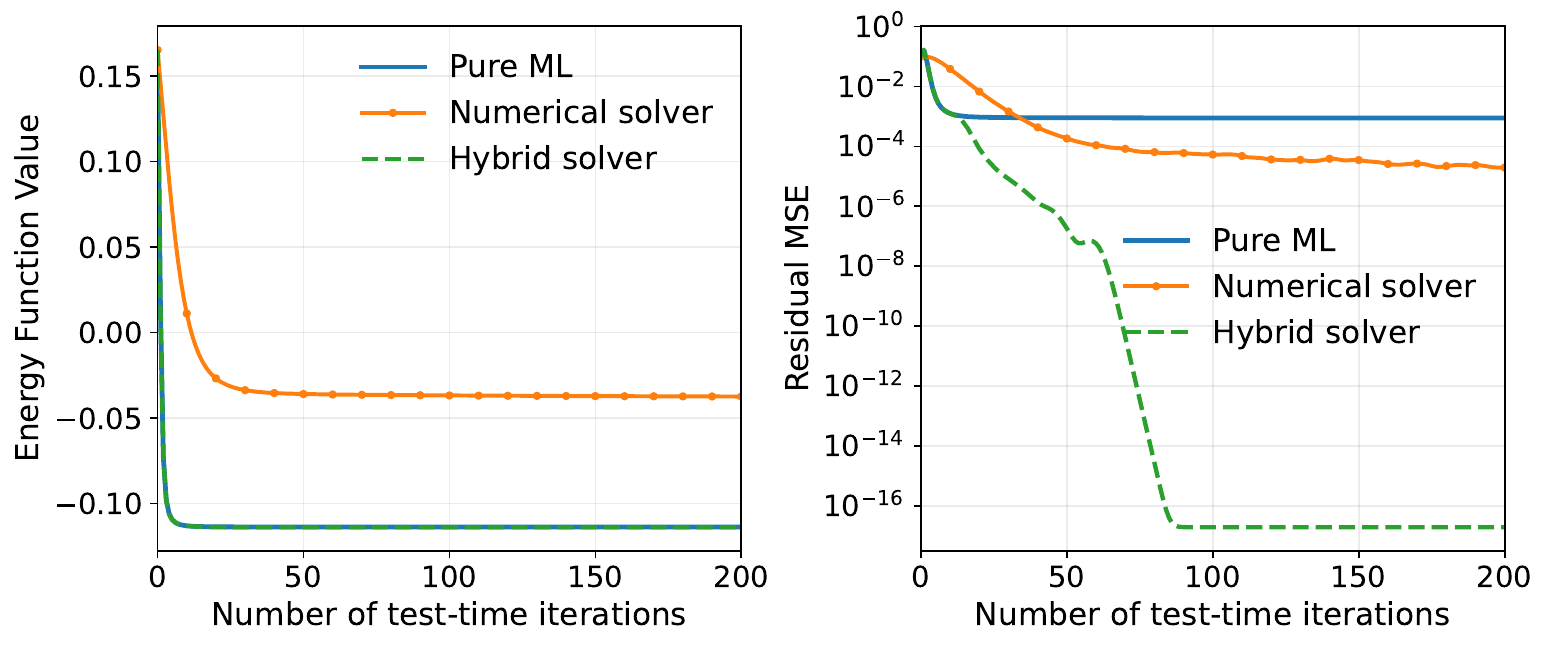}

  \caption{Convergence comparison for \emph{pure ML}, the \emph{numerical
  solver}, and the \emph{hybrid solver}.}
  \label{fig:ac-n64-ai-num-hybrid}
\end{figure}

\subsection{Additional Experiments: Resolution Scaling}

We extend the main Allen--Cahn comparison beyond the \(64\times64\) setting to
\(96\times96\) and \(128\times128\) grids. The physical problem and all training
hyperparameters are kept fixed. To fit memory, we use 32 retained Fourier modes
at all resolutions and reduce the batch size from 16 at \(64\times64\), to 8 at
\(96\times96\), and 4 at \(128\times128\).

\textbf{Result.}
Table~\ref{tab:fixedeps_resolution_implicit_plus_supervised} shows that the
same qualitative pattern persists across resolutions. IMEX-label fitting
collapses to one solution cluster and gives the worst physics errors. Energy-only
training remains the most accurate but also produces essentially one cluster.
Increasing the initial diversity weight \(\lambda_{\rm init}\) can recover many
distinct final-state clusters, especially at larger weights, but this generally
comes with worse residual error and energy. Thus the accuracy--diversity tradeoff
observed in the main \(64\times64\) experiment is not an artifact of a single
grid resolution.

\begin{table*}
\centering
\caption{Resolution extension of the Allen--Cahn study. We report mean best residual MSE, mean best energy, and the average number of distinct final-state clusters over 20 trajectories per forcing instance. Lower error/energy is better; larger \# Sols indicates greater diversity.}
\label{tab:fixedeps_resolution_implicit_plus_supervised}
\small
\setlength{\tabcolsep}{4pt}
\begin{tabular}{llrrrrr}
\toprule
Grid & Training obj. & Train error & Test error & Train energy & Test energy & \# Sols \\
\midrule
$96\times 96$ & IMEX label
& $9.51 \times 10^{-2}$ & $9.53 \times 10^{-2}$
& $1.38 \times 10^{-1}$ & $1.37 \times 10^{-1}$ & $1.00$ \\
$96\times 96$ & Energy only
& $1.55 \times 10^{-3}$ & $1.54 \times 10^{-3}$
& $-1.15 \times 10^{-1}$ & $-1.16 \times 10^{-1}$ & $1.04$ \\
$96\times 96$ & Div. $\lambda_{\rm init}=0.5$
& $1.47 \times 10^{-3}$ & $1.46 \times 10^{-3}$
& $-1.15 \times 10^{-1}$ & $-1.16 \times 10^{-1}$ & $1.11$ \\
$96\times 96$ & Div. $\lambda_{\rm init}=1.2$
& $1.27 \times 10^{-3}$ & $1.26 \times 10^{-3}$
& $-1.15 \times 10^{-1}$ & $-1.16 \times 10^{-1}$ & $1.14$ \\
$96\times 96$ & Div. $\lambda_{\rm init}=2.0$
& $1.45 \times 10^{-2}$ & $1.45 \times 10^{-2}$
& $-5.89 \times 10^{-2}$ & $-5.96 \times 10^{-2}$ & $19.74$ \\
\midrule
$128\times 128$ & IMEX label
& $9.26 \times 10^{-2}$ & $9.34 \times 10^{-2}$
& $1.37 \times 10^{-1}$ & $1.36 \times 10^{-1}$ & $1.00$ \\
$128\times 128$ & Energy only
& $2.24 \times 10^{-3}$ & $2.22 \times 10^{-3}$
& $-1.13 \times 10^{-1}$ & $-1.14 \times 10^{-1}$ & $1.01$ \\
$128\times 128$ & Div. $\lambda_{\rm init}=0.5$
& $1.88 \times 10^{-3}$ & $1.85 \times 10^{-3}$
& $-1.13 \times 10^{-1}$ & $-1.14 \times 10^{-1}$ & $1.05$ \\
$128\times 128$ & Div. $\lambda_{\rm init}=1.2$
& $1.97 \times 10^{-3}$ & $1.94 \times 10^{-3}$
& $-1.12 \times 10^{-1}$ & $-1.14 \times 10^{-1}$ & $1.13$ \\
$128\times 128$ & Div. $\lambda_{\rm init}=2.0$
& $2.69 \times 10^{-3}$ & $2.70 \times 10^{-3}$
& $-1.11 \times 10^{-2}$ & $-1.12 \times 10^{-2}$ & $19.98$ \\
\bottomrule
\end{tabular}
\end{table*}

\end{document}